\documentclass[sigconf]{acmart}

\AtBeginDocument{%
  \providecommand\BibTeX{{%
    \normalfont B\kern-0.5em{\scshape i\kern-0.25em b}\kern-0.8em\TeX}}}

\setcopyright{usgov}

\usepackage{graphicx} 
\usepackage{color}
\usepackage[ruled,vlined]{algorithm2e}
\usepackage{amsthm}
\usepackage{float}
\usepackage{hyperref} 
\usepackage{subcaption}
\usepackage{placeins}
\usepackage{flafter}
\usepackage{afterpage}
\usepackage[utf8]{inputenc}

\restylefloat{figure}

\newcommand{\longversion}[1]{\color{black}{#1}\color{black}}
\newcommand{\shortversion}[1]{}

\theoremstyle{plain}
\newtheorem{theorem}{Theorem}
\newtheorem{lemma}[theorem]{Lemma}

\newtheorem{definition}[theorem]{Definition}

\theoremstyle{remark}

\def\Lap{\text{Lap}}
\newcommand{\eps}{\varepsilon}
\newcommand{\bracs}[1]{\langle #1\rangle}
\newcommand{\norm}[1]{\|#1\|}
\newcommand{\E}{\mathbb{E}}

\def\gs{\text{GS}}
\def\bx{\textbf{x}}
\def\by{\textbf{y}}
\def\bz{\textbf{z}}

\def\be{\textbf{e}}
\def\ncov{\text{ncov}}
\def\nvar{\text{nvar}}

\def\var{\text{var}}
\def\ones{\mathbf{1}}
\def\calY{\mathcal{Y}}
\def\calR{\mathcal{R}}
\def\calN{\mathcal{N}}
\def\calX{\mathcal{X}}
\def\calZ{\mathcal{Z}}
\def\reals{{\mathbb R}}

\def\xnew{x_{new}}
\def\pxnew{p_{\xnew}}
\def\ptf{p_{25}}
\def\psf{p_{75}}
\def\hatpxnew{\widehat{p}_{\xnew}}
\def\hatptf{\widehat{p}_{25}}
\def\hatpsf{\widehat{p}_{75}}
\def\tildepxnew{\widetilde{p}_{\xnew}}
\def\tildeptf{\widetilde{p}_{25}}
\def\tildepsf{\widetilde{p}_{75}}
\def\hatsigma{\widehat{\sigma}}

\def\ci{\hat{C}(68)}
\def\citrue{\hat{C}_{\text{true}}}
\def\cptrue{C_{\text{true}}}

\newcommand\data{\ensuremath{\{(x_i, y_i)\}_{i=1}^n \in ([0,1] \times [0,1])^n}}

\DeclareMathOperator*{\Exp}{Exp}

\DeclareMathOperator*{\Var}{Var}

\DeclareMathOperator*{\argmin}{argmin}

\SetKwInput{KwInput}{Input}
\SetKwInput{KwPrivacyparams}{Privacy params}  
\SetKwInput{KwHyperparams}{Hyperparams} 
\SetKwInput{KwOutput}{Output}              

\begin{document}

\title{Differentially Private Simple Linear Regression}

\author{Daniel Alabi}
\affiliation{%
  \institution{Harvard John A. Paulson School of Engineering and Applied Sciences}
  \streetaddress{29 Oxford St}
}

\author{Audra McMillan}
\affiliation{%
  \institution{Department of Computer Science, Boston University}
  \streetaddress{111 Cummington Mall}
  }
 \affiliation{
 \institution{Khoury College of Computer Sciences, Northeastern University}
 \streetaddress{805 Columbus Ave}
 }

\author{Jayshree Sarathy}
\affiliation{%
  \institution{Harvard John A. Paulson School of Engineering and Applied Sciences}
  \streetaddress{29 Oxford St}
}

\author{Adam Smith}
\affiliation{
\institution{Department of Computer Science, Boston University}
  \streetaddress{111 Cummington Mall}
  }

\author{Salil Vadhan}
\affiliation{%
  \institution{Harvard John A. Paulson School of Engineering and Applied Sciences}
  \streetaddress{29 Oxford St}
}

\renewcommand{\shortauthors}{Alabi et al.}

\begin{abstract}
Economics and social science research often require analyzing datasets of sensitive personal information at fine granularity, with models fit to small subsets of the data. Unfortunately, such fine-grained analysis can easily reveal sensitive individual information. We study algorithms for simple linear regression that satisfy \textit{differential privacy}, a constraint which guarantees that an algorithm's output reveals little about any individual input data record, even to an attacker with arbitrary side information about the dataset. 
We consider the design of differentially private algorithms for simple linear regression for small datasets, with tens to hundreds of datapoints, which is a particularly challenging regime for differential privacy. Focusing on a particular application to small-area analysis in economics research, we study the performance of a spectrum of algorithms we adapt to the setting. We identify key factors that affect their performance, showing through a range of experiments that algorithms based on robust estimators (in particular, the Theil-Sen estimator)
perform well on the smallest datasets, but that other more standard algorithms do better as the dataset size increases.
\end{abstract}



\maketitle

\section{Introduction}

The analysis of small datasets, with sizes in the dozens to low hundreds, is crucial in many social science applications. For example, neighborhood-level household income, high school graduation rate, and incarceration rates are all studied using small datasets that contain sensitive, and often protected, data~(e.g., \cite{CFHJP18}). 
However, the release of statistical estimates based on these data 
quantities---if too many and too accurate---can allow reconstruction of the original dataset~\cite{DN03}. The possibility of such attacks led to differential privacy \cite{DMNS06}, a rigorous mathematical definition used to quantify privacy loss. Differentially private (DP) algorithms limit the information that is leaked about any particular individual by introducing random distortion. The amount of distortion, and its effect on utility, are most often studied for large datasets, using asymptotic tools. 
When datasets are small, one has to be very careful when calibrating differentially private statistical estimates to preserve utility.

In this work, we focus on simple (i.e. one-dimensional) linear regression and show that this prominent statistical task can have accurate differentially private algorithms even on small datasets. Our goal is 
to provide insight and guidance into how to choose a DP algorithm for simple linear regression in a variety of realistic parameter regimes. 

Even without a privacy constraint, small sample sizes pose a problem for traditional statistics, since the variability from sample to sample, called the \emph{sampling error}, can overwhelm the signal about the underlying trend.
A reasonable concrete goal for a DP mechanism, then, is that it not introduce substantially \emph{more} uncertainty into the estimate. Specifically, we  compare the noise added in order to maintain privacy to the standard error of the (nonprivate) OLS estimate. Our experiments indicate that for a wide range of realistic datasets and moderate values of the privacy parameter, $\eps$, it is possible to choose a DP linear regression algorithm that introduces distortion less than the standard error. In particular, in our motivating use-case of the Opportunity Atlas~\cite{Chetty14}, described below, we design a differentially private algorithm that matches or outperforms the heuristic method currently deployed, which does not formally satisfy differential privacy.

\subsection{Problem Setup}

\subsubsection{Simple Linear Regression}
\label{sec:problem}

In this paper we consider the most common model of linear regression: one-dimensional linear regression with homoscedastic noise.
This model is defined by a slope $\alpha\in\mathbb{R}$, an intercept $\beta\in\mathbb{R}$, 
a noise distribution $F_e$ with mean $0$ and variance $\sigma_e^2$,
and the relationship \[\by = \alpha\cdot \bx + \beta + \be\] where $\bx,\by\in\mathbb{R}^n$.
Our data consists of $n$ pairs, $(x_i,y_i)$, where 
$x_i$ is the explanatory variable, $y_i$ is the response variable, and
each random variable $e_i$ is draw i.i.d. from  the distribution $F_e$. 
The goal of simple linear regression is to estimate $\alpha$ and $\beta$ given the data $\{(x_1, y_1), \cdots, (x_n, y_n)\}$.

Let $\bx = (x_1, \ldots, x_n)^T$, $\by = (y_1, \ldots, y_n)^T$. The Ordinary Least Squares (OLS) solution to the simple linear regression problem is the solution to the following optimization problem: 
\begin{equation}\label{OLSopt}
(\hat\alpha, \hat\beta)=\arg\min_{\alpha,\beta} \|\by-\alpha\bx-\beta\textbf{1}\|_2.
\end{equation}
It is the most commonly used solution in practice since it is the maximum likelihood estimator when $F_e$ is Gaussian, and it has a closed form solution.
We define the following empirical quantities: 
\begin{align*}
    \bar{x} = \frac{1}{n}\sum_{i=1}^n x_i,\;&\;\; \bar{y} = \frac{1}{n}\sum_{i=1}^n y_i\\
    \ncov(\bx, \by) &= \langle \bx-\bar{x}\mathbf{1}, \by-\bar{y}\mathbf{1}\rangle,\\
    \nvar(\bx) &= \langle \bx-\bar{x}\mathbf{1}, \bx-\bar{x}\mathbf{1}\rangle = n\cdot \var(\bx).
\end{align*} 
Then
\begin{equation}\label{closedform}
\hat\alpha = \frac{\ncov(\bx, \by)}{\nvar(\bx)}\;\; \text{and} \;\; \hat\beta = \bar{y} - \hat\alpha\bar{x}
\end{equation}
 this paper, we focus on predicting the (mean of the) response variable $y$ at a single value of the explanatory variable $x$.  For $\xnew \in[0,1]$, the 
\emph{prediction} at $\xnew$ is defined as: 
\[\pxnew = \alpha \xnew +\beta \]
Let $\hatpxnew$ be the estimate of the prediction at $\xnew$ computed using the OLS estimates $\hat\alpha$ and $\hat\beta$. The quantity $\hatpxnew$ is a random variable, where the randomness is due to the sampling process. The standard error 
$\hatsigma(\hatpxnew)$ is an estimate of the standard deviation of $\hatpxnew$. 
If we assume that the noise $F_e$ is Gaussian, then we can compute the standard error as
follows~(see \cite{Xin09}, for example): \begin{equation}\label{standarderror}
\hatsigma(\hatpxnew) = \frac{\| \by-\hat\alpha \bx-\hat\beta\|_2}{\sqrt{n-2}}\sqrt{\frac{1}{n}+\frac{(\xnew-\bar{x})^2}{n \var(\bx)}}. \end{equation} 

It can be shown that the variance of $(\hatpxnew - \pxnew)/\hatsigma(\hatpxnew)$ approaches $1$ as $n$ increases. 

\subsubsection{Differential Privacy}

In this section, we define the notion of privacy that we are using: differential privacy (DP). Since our algorithms often include hyperparameters, we include a definition of DP for algorithms that take as input not only the dataset, but also the desired privacy parameters and any required hyperparameters. Let $\mathcal{X}$ be a data universe
and $\mathcal{X}^n$ be the space of datasets. Two datasets $d, d' \in \mathcal{X}^n$ are neighboring, denoted $d \sim d'$, if they differ on a single record. 
Let $\mathcal{H}$ be a hyperparameter space and $\mathcal{Y}$ be an output space.

\begin{definition}[$(\eps, \delta)$-Differential Privacy~\cite{DMNS06}]\label{def:DP}
A randomized mechanism $M: \mathcal{X}^n \times \mathbb{R}_{\geq 0} \times [0, 1] \times \mathcal{H} \rightarrow \mathcal{Y}$ is \emph{differentially private} if for all datasets $d \sim d' \in \mathcal{X}^n$, privacy-loss parameters $\eps\geq 0, \delta\in[0, 1]$, $\textrm{hyperparams} \in \mathcal{H}$, and events $E$,
\begin{align*} \label{def:dp-with-inputs}
    &\Pr[M(d, \eps, \delta, \text{hyperparams}) \in E] \\
    &\leq e^\eps \cdot\Pr[M(d', \eps, \delta, \text{hyperparams}) \in E] + \delta,
\end{align*}
where the probabilities are taken over the random coins of $M$.

\end{definition}
For strong privacy guarantees, the privacy-loss parameter is typically taken to be a small constant less than $1$ (note that $e^\eps \approx 1+\eps$ as $\eps \rightarrow 0$), but we will sometimes consider larger constants such as $\eps = 8$ to match our motivating application (described in Section~\ref{OIusecase}).

The key intuition for this definition is that the distribution of outputs on input dataset $d$ is almost indistinguishable from the distribution on outputs on input dataset $d'$. Therefore, given the output of a differentially private mechanism, it is impossible to confidently determine whether the input dataset was $d$ or $d'$.

\subsubsection{Other Notation}
The DP estimates will be denoted by $\tildepxnew$ and the OLS estimates by
$\hatpxnew$. We will focus on data with values bounded between 0 and 1, so $0\le x_i, y_i\le 1$ for $i=1,\ldots, n$. 

We will be primarily concerned with the predicted values at $\xnew =0.25$ and $0.75$, which for ease of notation we denote as 
$\ptf$ and $\psf$, respectively. Correspondingly, we will use
$\hatptf, \hatpsf$ to denote the OLS estimates of the predicted values and
$\tildeptf, \tildepsf$ to denote the DP estimates.
Our use of the 25th and 75th percentiles is motivated by
    the Opportunity Atlas tool~\cite{CFHJP18}, described in Section~\ref{OIusecase}, which releases estimates of $\ptf$ and $\psf$ for certain regressions done for every census tract in each state. 

\subsubsection{Error Metric}
We will be concerned primarily with empirical performance measures.
In particular, we will restrict our focus to high probability error bounds that can be accurately computed empirically through Monte Carlo experiments. The question of providing tight theoretical error bounds for DP linear regression is an important one, which we leave to future work. 
Since the relationship between the OLS estimate $\hatpxnew$ and the true value $\pxnew$ is well-understood, we focus on measuring the difference between the private estimate $\tildepxnew$ and $\hatpxnew$, which is something we can measure experimentally on real datasets.
Specifically, we define the \emph{prediction error} at $\xnew$ to be
    \(|\tildepxnew - \hatpxnew|.\)

For a dataset $d$, $\xnew \in [0,1]$, and $q \in [0,100]$, we define the \emph{$q\%$ error bound} as
\[C(q)(d) = \min\left\{c: \mathbb{P}(|\tildepxnew-\hatpxnew|\leq c)\ge \frac{q}{100}\right\},\] where the dataset $d$ is fixed, and the probability is taken over the randomness in the DP algorithm. 

We empirically estimate $C(q)$ by running many trials of the algorithm on the same dataset $d$:
        \[\hat{C}(q)(d) = \min\left\{c: \text{for at least } q\%\text{ of trials}, |\tildepxnew-\hatpxnew| \leq c\right\}.\]
    We term $\hat{C}(q)(d)$ the \emph{$q\%$ empirical error bound}. 
We will often drop the reference to $d$ from the notation. This error metric only accounts for the randomness in the algorithm, not the sampling error.  

When the ground truth is known (eg. for synthetically generated data), we can compute error bounds compared to the ground truth, rather than to the non-private OLS estimate. So, let $\cptrue(q)(d)$ and $\citrue(q)(d)$ be similar to the error bounds described earlier, except that the prediction error is measured as $|\tildepxnew-\pxnew|$. This error metric accounts for the randomness in both the sampling and the algorithm. 
When we say \emph{the noise added for privacy is less than the sampling error}, we are referring to the technical statement that $\hat{C}(68)$ is less than the 
standard error, $\hatsigma(\hatpxnew)$.

\subsection{Motivating Use-Case: Opportunity Atlas}\label{OIusecase}

The Opportunity Atlas, designed and deployed by the economics research group Opportunity Insights, is an interactive tool designed to study the link between the neighbourhood a child grows up in and their prospect for economic mobility \cite{CFHJP18}. The tool provides valuable insights to researchers and policy-makers, and is built from Census data, protected under Title 13 
and authorized by the Census Bureau’s Disclosure Review Board, 
linked with federal income tax returns from the US Internal Revenue Service. 
The Atlas provides individual statistics on each census tract in the country, with tract data often being refined by demographics to contain only a small subset of the individuals who live in that tract;
\longversion{for example, black male children with low parental income. }
 The resulting datasets typically contain  100 to 400 datapoints, but can be as small as 30 datapoints. 
A separate regression estimate is released in each of these small geographic areas.  The response variable $y_i\in[0,1]$ is the child’s income percentile at age 35 and the explanatory variable $x_i\in[0,1]$ is the parent’s income percentile,
each
with respect to the national income distribution. The coefficient $\alpha$ 
in the model
$y_i = \alpha \cdot x_i + \beta + e_i$
is a measure of economic mobility for that particular Census tract and demographic.
The small size of the datasets used in the Opportunity Atlas are the result of Chetty et al.'s motivation to study inequality at the neighbourhood level:
\longversion{
``the estimates permit precise targeting of policies
to improve economic opportunity by uncovering specific neighborhoods where certain subgroups
of children grow up to have poor outcomes. Neighborhoods matter at a very granular level:
conditional on characteristics such as poverty rates in a child’s own Census tract, characteristics
of tracts that are one mile away have little predictive power for a child’s outcomes"~\cite{CFHJP18}. 
}

\subsection{Robustness and DP Algorithm Design}\label{robustnessalgodesign}
Simple linear regression is one of the most fundamental statistical tasks with well-understood convergence properties in the non-private literature. However, finding a differentially private estimator for this task that is accurate across a range of datasets and parameter regimes is surprisingly nuanced. As a first attempt, one might consider
the global sensitivity \cite{DMNS06}: 

\begin{definition}[Global Sensitivity]
For a query $q:~\calX^n~\rightarrow~\reals^k$, the \textbf{global sensitivity} is
\[
GS_q = \max_{\bx \sim \bx'}\|q(\bx) - q(\bx')\|_1.
\]
\end{definition}

One can create a differentially private mechanism by adding noise proportional to $GS_q/\eps$. Unfortunately, the global sensitivity of $\ptf$ and $\psf$ are both infinite (even though we consider bounded  $\bx,\by\in[0,1]^n$, the point estimates $\ptf$ and $\psf$ are unbounded). For the type of datasets that we typically see in practice, however, changing one datapoint does not result in a major change in the point estimates. For such datasets, where the point estimates are reasonably stable, one might hope to take advantage of the local sensitivity:

\begin{definition}[Local Sensitivity~\cite{NRS07}] The \textbf{local sensitivity} of a query $q:~\calX^n~\rightarrow~\reals^k$ with respect to a dataset $\bx$ is 
\[
LS_q(\bx) = \max_{\bx\sim\bx'} \norm{q(\bx) - q(\bx')}_1.
\]
\end{definition}

Adding noise proportional to the local sensitivity is typically \textbf{not} differentially private, since the local sensitivity itself can reveal information about the underlying dataset. To get around this, one could try to
add noise that is larger, but not too much larger, than the local sensitivity. As DP requires that the amount of noise added cannot depend too strongly on the particular dataset, DP mechanisms of this flavor often involve calculating the local sensitivity on neighbouring datasets. So far, we have been unable to design a computationally feasible algorithm for performing the necessary computations for OLS. Furthermore, computationally feasible upper bounds for the local sensitivity have, so far, proved too loose to be useful.

The Opportunity Insights (OI) algorithm takes a heuristic approach by adding noise proportional to a non-private, heuristic upper bound on the local sensitivity of data from tracts in any given state. 
However, their heuristic approach does not satisfy the formal requirements of differential privacy, leaving open the possibility that there is a realistic attack.

The OI algorithm incorporates a ``winsorization'' step in their estimation procedure (e.g. dropping bottom and top 10\%
of data values). This sometimes has the effect of greatly reducing the local sensitivity (and also their upper bound on
it) due to the possible removal of outliers. This suggests that for finding an effective differentially private
algorithm, we should consider differentially private analogues of \emph{robust} linear regression methods rather than
of OLS. Specifically, we consider \emph{Theil-Sen}, a robust estimator for linear regression proposed by
Theil~\cite{Theil50} and further developed by Sen~\cite{Sen68}.
Similar to the way in which the median is less sensitive to changes in the data than the mean, the Theil-Sen estimator is more robust to changes in the data than OLS . 

In this work, we consider three differentially private  algorithms based on both robust and non-robust methods:

\begin{itemize}
    \item \textbf{\texttt{NoisyStats}} is the DP mechanism that most closely mirrors OLS. It involves perturbing the sufficient statistics $\ncov(\bx,\by)$ and $\nvar(\bx)$ from the OLS computation. This algorithm is inspired by the ``Analyze Gauss'' technique~\cite{DworkTT014}, although the noise we add ensures pure differential
    privacy rather than approximate differential privacy. 
    \texttt{NoisyStats} has two main benefits: it is as computationally efficient as its non-private analogue, and it allows us to release DP versions of the sufficient statistics with no extra privacy cost.
    
\item \textbf{\texttt{DPTheilSen}} is a DP version of Theil-Sen. The non-private estimator computes the $\ptf$ estimates based on the lines defined by all pairs of
    datapoints $(x_i, y_i), (x_j, y_j)\in [0, 1]^2$ for all $i \neq j\in[n]$, then outputs the median of these pairwise estimates. To create a differentially private version, we replace the median computation with a differentially private median algorithm.
    We will consider three DP versions of this algorithm which use different DP median algorithms: \texttt{DPExpTheilSen, DPWideTheilSen,} and \texttt{DPSSTheilSen}. 
    We also consider more computationally efficient variants that pair points according to one or more random matchings,
    rather than using all ${n \choose 2}$ pairs. A DP algorithm obtained by using one matching was previously
    considered by Dwork and Lei~\cite{DworkL09} (their ``Short-Cut Regression Algorithm'').
    Our algorithms can be viewed as updated versions, reflecting improvements in DP median estimation since~\cite{DworkL09}, as well as incorporating benefits accrued by considering more than one matching or all
    ${n\choose 2}$ pairs.
    
\item \textbf{\texttt{DPGradDescent}} is a DP mechanism that uses DP gradient descent to solve the convex optimization problem that defines OLS: $\argmin_{\alpha,\beta} \|\by-\alpha\bx-\beta\mathbf{1}\|_2.$ We use the private stochastic gradient descent
technique proposed by Bassily et. al in~\cite{BST14}. Versions that satisfy
pure, approximate, and zero-concentrated differential privacy are considered.
\end{itemize}

\subsection{Our Results}

Our experiments indicate that for a wide range of realistic datasets, and moderate values of $\eps$, it is possible to choose a DP linear regression algorithm where the error due to privacy is less than the standard error. In particular, in our motivating use-case of the Opportunity Atlas, we can design a differentially private algorithm that outperforms the heuristic method used by the Opportunity Insights team. This is promising, since the error added by the heuristic method was deemed acceptable by the Opportunity Insights team for deployment of the tool, and for use by policy makers.
One particular differentially private algorithm of the robust variety, called  \texttt{DPExpTheilSen},
emerges as the best algorithm in a wide variety of settings for this small-dataset regime.

Our experiments reveal three main settings where an analyst should consider alternate algorithms: 
\begin{itemize}
    \item When $\eps n\var(\bx)$ is 
large and $\sigma_e^2$ is large,
a DP algorithm \texttt{NoisyStats} that simply perturbs the Ordinary Least
Squares (OLS) sufficient statistics, $\nvar(\bx)$ and $\ncov(\bx, \by)$, performs well. This algorithm is preferable in this setting since it is more computationally efficient, and allows for the release of noisy sufficient statistics without any additional privacy loss.
\item When the standard error $\hatsigma(\hatpxnew)$ is very small, \newline \texttt{DPExpTheilSen} can perform poorly. In this setting, one should switch to a different DP estimator based on Theil-Sen. We give two potential alternatives, which are both useful in different situations.
\item The algorithm \texttt{DPExpTheilSen} requires as input a range in which to search for the output predicted value. If this range is very large and $\eps$ is small ($\eps\ll1$) then \texttt{DPExpTheilSen} can perform poorly, so it is good to use other algorithms like \texttt{NoisyStats} that do not require a range for the output (but do require that the input variables $x_i$ and $y_i$ are bounded, which is not required for \texttt{DPExpTheilSen}). 
\end{itemize}

The quantity $\eps n\var(\bx)$ is connected to the size of the dataset, how concentrated the independent variable of the data is and how private the mechanism is. Experimentally, this quantity has proved to be a good indicator of the relative performance of the DP algorithms. 
Roughly, when $\eps n\var(\bx)$ is small, the OLS estimate can be very sensitive to small changes in the data, and thus we recommend switching to 
differentially private versions of the Theil-Sen estimator.
In the opposite regime, when $\eps n\var(\bx)$ is large, \texttt{NoisyStats} typically suffices and is a simple non-robust method to adopt in practice. 
In this regime the additional noise added for privacy by \texttt{NoisyStats} can be less than the difference between the non-private OLS and Theil-Sen point estimates.
\longversion{The other non-robust algorithm we consider, \texttt{DPGradDescent},
may be a suitable replacement for \texttt{NoisyStats} depending on
the privacy model used (i.e. pure, approximate, or zero-concentrated DP). Our comparison of \texttt{NoisyStats} and \texttt{DPGradDescent} is not
comprehensive, but we find that any performance advantages of \texttt{DPGradDescent} over \texttt{NoisyStats}
appear to be small in the regime where the non-robust methods outperform the robust ones. \texttt{NoisyStats} is simpler and has fewer hyperparameters, however, so we find that it may be preferable in practice.
}

In addition to the quantity $\eps n\var(\bx)$, the magnitude of noise in the dependent variable effects the relative performance of the algorithms. When the dependent variable is not very noisy (i.e. $\sigma_e^2$ is small), the Theil-Sen-based estimators perform better since they are better at leveraging a strong linear signal in the data.

These results show that even in this one-dimensional setting, the story is already quite nuanced. 
Indeed, which algorithm performs best depends on properties of the dataset, such as $n \var(\bx)$, which cannot be directly used without violating differential privacy. So, one has to make the choice based on guesses (eg. using similar public datasets) or develop differentially private methods for selecting the algorithm, a problem which we leave to future work.
Moreover, most of our methods come with hyperparameters that govern their behavior. 
How to optimally choose these parameters is still an open problem. In addition, we focus on outputting accurate point estimates, rather than confidence intervals. Computing confidence intervals is an important direction for future work.

\subsection{Other Related Work}\label{relatedwork} 
Linear regression is one of the most prevalent statistical methods in the social sciences, and hence has been studied previously in the differential privacy literature. These works have included both theoretical analysis and experimental exploration, with the majority of work focusing on large datasets.

One of our main findings --- that robust estimators perform better than parametric estimators in the differentially private setting, even when the data come from a parametric model --- corroborate insights by~\cite{DworkL09} with regard to the connection between robust statistics and differential privacy, and by~\cite{CKSBG19} in the context of hypothesis testing.

Other systematic studies of DP linear regression have been performed by Sheffet~\cite{Sheffet17} and Wang~\cite{Wang18}.
Sheffet \cite{Sheffet17} considered differentially private
ordinary least squares methods and estimated confidence intervals for the regression coefficients. He assumes normality of the explanatory variables, while we do not make any distributional assumptions on our covariates.

Private linear regression in the high-dimensional settings is
studied by Cai et al.~\cite{CWZ19} and Wang~\cite{Wang18}.
Wang~\cite{Wang18} considered private ridge regression, where the ridge parameter is adaptively
and privately chosen, using techniques similar to output
perturbation~\cite{CMS11}. 
These papers present methods and experiments for
high dimensional data (the datasets used contain at least 13 explanatory variables), whereas we are concerned with the one-dimensional setting. 
We find that even the one-dimensional setting the choice of optimal algorithm is already quite complex.

A Bayesian approach to DP
linear regression is taken by Bernstein and Sheldon~\cite{GarrettS19}. Their method uses sufficient statistic perturbation (similar to our
\texttt{NoisyStats} algorithm) for private release and Markov Chain Monte Carlo sampling 
over a posterior of the regression coefficients.
Their Bayesian approach can produce tight credible intervals for the
regression coefficients, but unlike ours, requires distributional assumptions on
both the regression coefficients and the underlying independent variables. 
In order to make private releases, 
we assume the data is bounded but make no other distributional assumptions on 
the independent variables.

\section{Algorithms}\label{sec:algorithms}

In this section we detail the practical differentially private algorithms we will
evaluate experimentally. Pseudocode for all efficient implementations of each algorithm described can be found in
the Appendix,
and real code can be found in our GitHub repository.
We will assume throughout that $(x_i,y_i)\in[0,1]^2$, as in the Opportunity Atlas, where it is achieved by preprocessing of the data. While we would ideally ensure differentially private preprocessing of the data, we will consider this outside the scope of this work.

\subsection{\texttt{NoisyStats}}
In \texttt{NoisyStats} (Algorithm~\ref{alg:ns}), we add Laplace noise, with standard deviation approximately $1/\eps$, to the OLS sufficient statistics, $\ncov(\bx, \by), \nvar(\bx)$, and then use the noisy sufficient statistics to compute the predicted values.  
Note that this algorithm fails if the denominator for the OLS estimator, the noisy version of $\nvar(\bx)$, becomes 0 or negative, in which case we output $\perp$ (failure). The probability of failure decreases as $\eps$ or
$\nvar(x)$ increases.
\shortversion{}
\longversion{}
\texttt{\texttt{NoisyStats}} is the simplest and most efficient algorithm that we will study. In addition, the privacy guarantee is maintained even if we additionally release the noisy statistics $\nvar(\bx)+L_1$ and $\ncov(\bx,\by)+L_2$, which may be of independent interest to researchers. We also note that the algorithm is biased due to dividing by a Laplacian distribution centered at $\nvar(\bx)$.
\longversion{

\begin{algorithm}
  \KwData{\data}
  \KwPrivacyparams{$\eps$}
  \KwHyperparams{n/a}
  Define $\Delta_1 = \Delta_2 = \left(1-1/n\right)$

  Sample $L_1 \sim \Lap(0, 3\Delta_1/\eps)$

  Sample $L_2 \sim \Lap(0, 3\Delta_2/\eps)$

  \If {$\nvar(\bx) + L_2 > 0$} {
    $\tilde\alpha = \frac{\ncov(\bx, \by) + L_1}{\nvar(\bx) + L_2}$
    
    $\Delta_3 = 1/n\cdot(1 + |\tilde\alpha|)$
    
    Sample $L_3 \sim \Lap(0, 3\Delta_3/\eps)$
    
    $\tilde\beta = (\bar{y} - \tilde\alpha\bar{x}) + L_3$
    
    $\tildeptf = 0.25\cdot \tilde\alpha + \tilde\beta$
    
    $\tildepsf = 0.75\cdot \tilde\alpha + \tilde\beta$
    
    \Return $\tildeptf, \tildepsf$
  } \Else {
    \Return $\perp$
  }
  \caption{\texttt{NoisyStats}: $(\eps, 0)$-DP Algorithm}  \label{alg:ns}
\end{algorithm}

\begin{lemma}\label{lem:ns}
Algorithm~\ref{alg:ns} (\texttt{NoisyStats}) is $(\eps, 0)$-DP.
\end{lemma}}

\subsection{DP TheilSen}\label{alg:dpts}

The standard Theil-Sen estimator is a \emph{robust} estimator for linear regression. It computes the $\ptf$ estimates based on the lines defined by all pairs of
datapoints $(x_i, y_i), (x_j, y_j)\in [0, 1]^2$ for all $i\neq j\in[n]$, then outputs the median of these pairwise estimates. To create a differentially private version, we can replace the median computation with a differentially private median algorithm.
We implement this approach using three DP median algorithms; two based on the exponential mechanism~\cite{McSherryT07} and one based on the smooth sensitivity of~\cite{NRS07} and the noise distributions of \cite{BunS19}.

In the ``complete" version of Theil-Sen, all pairwise estimates are included in the final median computation. A similar algorithm can be run on the point estimates computed using $k$ random matchings of the $(x_i,y_i)$ pairs. 
The case $k=1$ amounts to the differentially private ``Short-cut Regression Algorithm'' proposed by Dwork and Lei~\cite{DworkL09}. This results in a more computationally efficient algorithm.

\longversion{ Furthermore, while in the non-private setting $k=n-1$ is the optimal choice, this is not necessarily true when we incorporate privacy. Each datapoint affects $k$ of the datapoints in $\bz^{(p25)}$ (using the notation from Algorithm~\ref{alg:ts-match}) out of the total $k\cdot n/2$ datapoints. In some private algorithms, the amount of noise added for privacy is a function of the fraction of points in $\bz^{(p25)}$ influenced by each $(x_i,y_i)$, which is independent of $k$ -- meaning that one should always choose $k$ as large as is computationally feasible. However, this is not always the case. Increasing $k$ can result in more noise being added for privacy resulting in a trade-off between decreasing the noise added for privacy (with smaller $k$) and decreasing the non-private error (with larger $k$). 
}

\longversion{While intuitively, for small $k$, one can think of using $k$ random matchings, we will actually ensure that no edge is included twice. We say the permutations $\tau_1, \cdots, \tau_{n-1}$ are a decomposition of the complete graph on $n$ vertices, $K_n$ into $n-1$ matchings if $\cup \Sigma_k = \{(x_i, x_j)\;|\; i,j\in[n]\}$ where $\Sigma_k=\{(x_{\tau_k(1)}, x_{\tau_k(2)}), \cdots, (x_{\tau_k(n-1)}, x_{\tau_k(n)})\} $. Thus, \texttt{DPTheilSen(n-1)Match}, referred to simply as \texttt{DPTheilSen}, uses all pairs of points. } We will focus mainly on $k=n-1$, which we will refer to simply as \texttt{DPTheilSen} and $k=1$, which we will refer to as \texttt{DPTheilSenMatch}. For any other $k$, we denote the algorithm as \texttt{DPTheilSenkMatch}. In the following subsections we discuss the different differentially private median algorithms we use as subroutines.
The pseudo-code for \texttt{DPTheilSenkMatch} is found in Algorithm~\ref{alg:ts-match}.

\longversion{

\begin{lemma}\label{lem:ts-match} 
If DPmed$(\bz^{(p25)}, \eps, (n, k, \text{hyperparameters})=\\\mathcal{M}(z^{(p25)}, \text{hyperparameters}))$ for some $(\eps/k, 0)$-DP mechanism $\mathcal{M}$, then Algorithm~\ref{alg:ts-match} (\texttt{DPTheilSenkMatch}) is $(\eps, 0)$-DP.
\end{lemma}

\begin{algorithm} 
  \KwData{\data}
  \KwPrivacyparams{$\eps$}
  \KwHyperparams{$n, k$, DPmed, hyperparams}
  
  $\bz^{(p25)}, \bz^{(p75)}~= [~]$
  
  Let $\tau_1, \cdots, \tau_{n-1}$ be a decomposition of $K_n$ into matchings.
  
  \For{$k$ iterations}{
  
  Sample (without replacement) $h\in[n-1]$

  \For{$0 \leq i < n-1, i = i + 2$} {
    $j = \tau_h(i)$
        
    $l = \tau_h(i+1)$
    
    \If {$(x_{l} - x_{j} \neq 0)$} {
        
        $s = (y_{l} - y_{j})/(x_{l} - x_{j})$
        
        $z^{(p25)}_{j,l} = s \left(0.25 - \frac{x_l + x_j}{2}\right) + \frac{y_l + y_j}{2}$
        
        $z^{(p75)}_{j,l} = s \left(0.75 - \frac{x_l + x_j}{2}\right) + \frac{y_l + y_j}{2}$

        Append $z^{(p25)}_{j, l}$ to $\bz^{(p25)}$ and $z^{(p75)}_{j, l}$ to $\bz^{(p75)}$
        }
  }
  }
  $\tilde{p}_{25} = \textrm{DPmed}\left(\bz^{(p25)}, \eps, (n, k, \text{hyperparams})\right)$
 
  $\tilde{p}_{75} = \textrm{DPmed}\left(\bz^{(p75)}, \eps, (n, k, \text{hyperparams})\right)$ 
  
  \Return $\tilde{p}_{25}, \tilde{p}_{75}$

  \caption{\texttt{DPTheilSenkMatch}: $(\eps, 0)$-DP Algorithm} \label{alg:ts-match}
\end{algorithm}

}

\subsubsection{DP Median using the Exponential Mechanism}

The first differentially private algorithm for the median that we will consider is an instantiation of the exponential mechanism~\cite{McSherryT07}, a differentially private algorithm designed for general optimization problems. The exponential mechanism is defined with respect to a utility function $u$, which maps (dataset, output) pairs to real values. For a dataset $\bz$, the mechanism aims to output a value $r$ that maximizes $u(\bz,r)$.

\begin{definition}[Exponential Mechanism~\cite{McSherryT07}] 
Given dataset $\bz \in \reals^n$ and the range of the outputs, $[r_l, r_u]$, the exponential mechanism outputs $r \in [r_l, r_u]$ with probability proportional to $\exp\left(\frac{\eps u(\bz, r)}{2 GS_u}\right)$, where \[GS_u=\max_{r\in[r_l,r_u]}\max_{\bz, \bz' \text{neighbors}} |u(\bz, r) - u(\bz', r)|.\]
\end{definition}

One way to instantiate the exponential mechanism to compute the median is by using 
the following utility function. Let 
\[
u(\bz, r) = - \left \vert
\#\text{above } r - \#\text{below } r
\right\vert 
\]
where \#above $r$ and \#below $r$ denote the number of datapoints in $\bz$ that are above and below $r$ in value respectively, not including $r$ itself.
\shortversion{}
 An example of the shape of the output distribution of this algorithm is given in Figure~\ref{diag:em}. An efficient implementation is given in the Appendix. 
\shortversion{}
\longversion{

}
\begin{figure}
\includegraphics[scale=0.2]{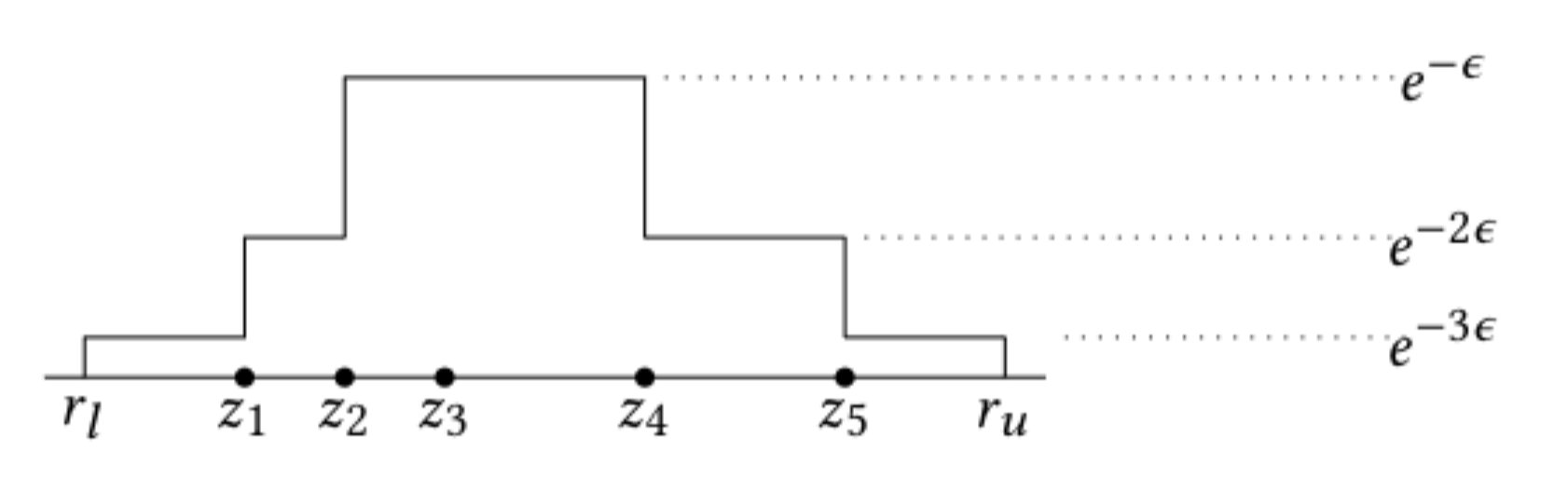}
\caption{Unnormalized distribution of outputs of the exponential mechanism for differentially privately computing the median of dataset $\bz$.}\label{diag:em}
\end{figure}
\shortversion{}
We will write \texttt{DPExpTheilSenkMatch} to refer to \texttt{DPTheilSenkMatch} where \shortversion{the DP median function } \longversion{DPmed } is the DP exponential mechanism described above with privacy parameter $\eps/k$. Again, we write \texttt{DPExpTheilSenMatch} when $k=1$ and \texttt{DPExpTheilSen} when $k=n-1$.
\longversion{
\begin{lemma}\label{lem:ets}
\texttt{DPExpTheilSenkMatch} is $(\eps, 0)$-DP. 
\end{lemma}}

\subsubsection{DP Median using Widened Exponential Mechanism}
When the output space is the real line, the standard exponential mechanism for the median has some nuanced behaviour when the data is highly concentrated. For example, imagine in Figure~\ref{diag:em} if all the datapoints coincided. In this instance, \texttt{DPExpTheilSen} is simply the uniform distribution on $[r_l, r_u]$, despite the fact that the median of the dataset is very stable. To mitigate this issue, we use a variation on the standard utility function. For a widening parameter $\theta>0$, the widened utility function is 
\[u(\bz, r) = -\min\left\{\left \vert
\#\text{above } a - \#\text{below } a
\right \vert \;:\; |a-r|\le\theta\right\} \]
where \#above $a$ and \#below $a$ are defined as before. 
This has the effect of increasing the probability mass around the median, as shown in Figure~\ref{diag:wem}.

\begin{figure}
\includegraphics[scale=0.2]{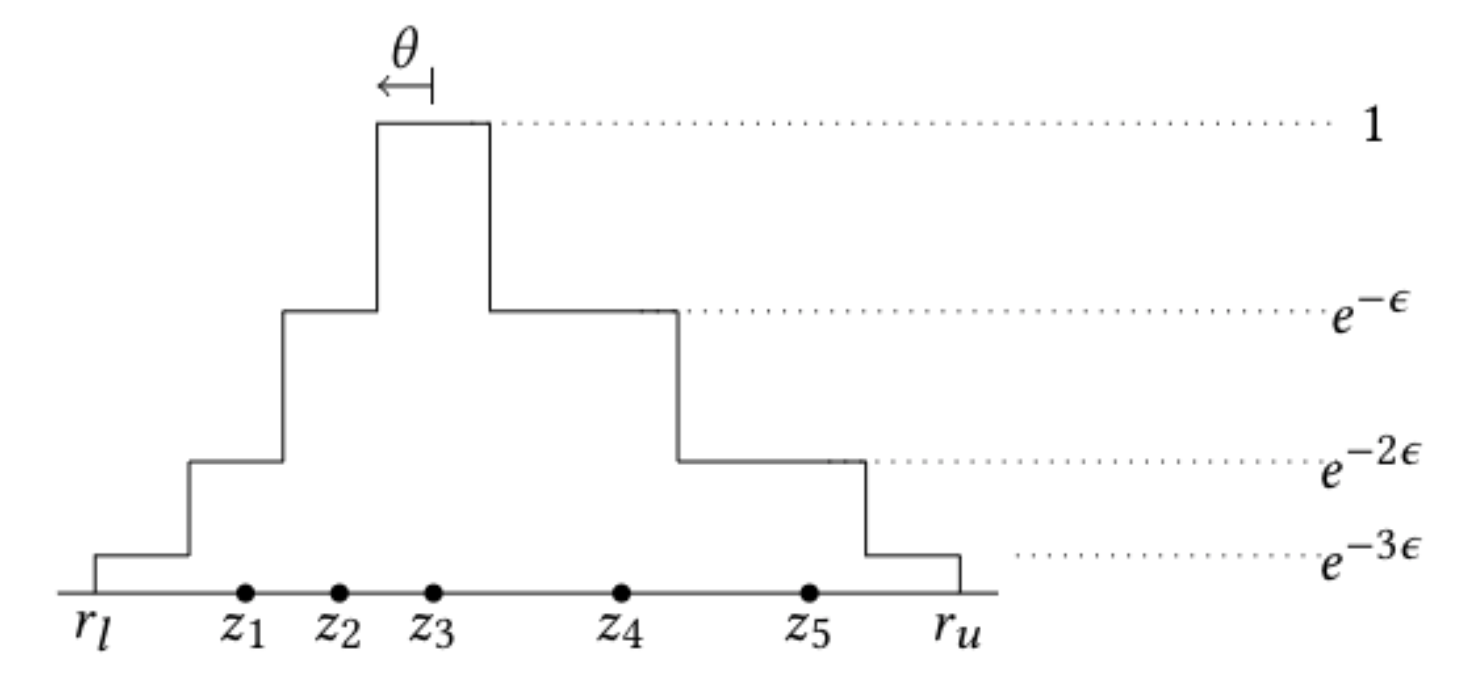}
\caption{Unnormalized distribution of outputs of the $\theta$-widened exponential mechanism for differentially privately computing the median of dataset $\bz$.}\label{diag:wem}
\end{figure}

The parameter $\theta$ needs to be carefully chosen. All outputs within $\theta$ of the median are given the same utility score, so $\theta$ represents a lower bound on the error. Conversely, choosing $\theta$ too small may result in the area around the median not being given sufficient weight in the sampled distribution. 
\longversion{Note that $\theta>0$ is only required when one expects the datasets $\bz^{p25}$ to be highly concentrated (for example when the dataset has strong linear signal). }
We defer the question of optimally choosing $\theta$ to future work. 

An efficient implementation of the $\theta$-widened exponential mechanism for the median can be found in
the Appendix\longversion{ as Algorithm~\ref{alg:wem}}. We will use \texttt{DPWideTheilSenkMatch} to refer to \texttt{DPTheilSenkMatch} where \shortversion{the DP median mechanism }\longversion{DPmed } is the $\theta$-widened exponential mechanism with privacy parameter $\eps/k$. Again, we use\\ \texttt{DPWideTheilSenMatch} when $k=1$ and \texttt{DPWideTheilSen} when $k=n-1$.

\longversion{
\begin{lemma}\label{lem:wts}
\texttt{DPWideTheilSenkMatch} is $(\eps, 0)$-DP.
\end{lemma}}

\subsubsection{DP Median using Smooth Sensitivity Noise Addition}
The final algorithm we consider for releasing a differentially private median adds noise scaled to the smooth sensitivity -- a smooth upper bound on the local sensitivity function. Intuitively, this algorithm should perform well when the datapoints are clustered around the median; that is, when the median is very stable. 
\begin{definition}[Smooth Upper Bound on $LS_f$~\cite{NRS07}] For $t > 0$, a function
$S_{f, t}:\calX^n\rightarrow\reals$ is a $t$-{\em smooth
upper bound on the local sensitivity} of a function $f: \calX^n \rightarrow \reals$ if:
$$
\forall \bz\in\calX^n: LS_f(\bz) \leq S_{f, t}(\bz);
$$
$$
\forall \bz, \bz'\in\calX^n, d(\bz, \bz') = 1: S_{f, t}(\bz) \leq e^t\cdot S_{f, t}(\bz').
$$

where $d(\bz, \bz')$ is the distance between datasets $\bz$ and $\bz'$. 
\label{def:upper}
\end{definition}

Let $\calZ_k:([0, 1]\times[0, 1])^n\rightarrow\reals^{kn/2}$ to denote
the function that transforms a set of point coordinates into estimates for each pair of 
points in our $k$ matchings\longversion{, so in the notation of Algorithm~\ref{alg:ts-match}, 
$\bz^{p25} = \calZ(\bx, \by)$}. 
The function that we are concerned with the smooth sensitivity of is med$\circ\calZ_k$\longversion{, which maps $(\bx, \by)$ to med$(\bz^{p25})$}. We will use the following  
smooth upper bound to the local sensitivity:
\begin{lemma}
\label{lem:ss-upper}
Let $z_1\leq z_2\leq \cdots \leq z_{2m}$ be a sorting of $\calZ_k(\bx, \by)$. Then
\begin{align*} 
    S_{\text{med}\circ\calZ, t}^k((\bx, \by))& \\
    =\max\Big\{&z_{m+k}-z_m, z_{m}-z_{m-k}, \\
    & \max_{l = 1, \ldots, n} \max_{s=0, \cdots, k(l+1)} e^{-lt} ( z_{m+s}-z_{m-(k(l+1)+s})\Big\},
\end{align*} 
is a $t$-smooth upper bound on the local sensitivity of med$\circ\calZ_k$.
\end{lemma}
\begin{proof}
Proof in the Appendix.
\end{proof}

The algorithm then adds noise proportional to $S_{\text{med}\circ\calZ, t}^k((\bx, \by))/\eps$ to med$\circ\calZ(\bx,\by)$. \longversion{A further discussion of the formula $S_{\text{med}\circ\calZ, t}^k((\bx, \by))$ and pesudo-code can be found in Appendix~\ref{sec:compute-ss}. } \shortversion{Pseudo-code is given in the Appendix.}
The noise is sampled from the Student's T distribution.
There are several other valid choices of noise distributions  (see~\cite{NRS07} and~\cite{BunS19}), but we found the Student's T distribution to be preferable as the mechanism remains stable across values of $\eps$ 

\longversion{
\begin{lemma}\label{lem:ssts}
\texttt{DPSSTheilSenkMatch} is $(\eps, 0)$-DP.
 \end{lemma}
 }

 \subsection{DP Gradient Descent}

 Ordinary Least Squares (OLS) for simple 1-dimensional linear
 regression is defined as the solution to the optimization problem \shortversion{in Equation~\ref{OLSopt}.} \longversion{\[\argmin_{\alpha, \beta} \sum_{i=1}^n (y_i-\alpha x_i-\beta)^2.\]}
 There has been an extensive line of work on solving convex optimization problems in a differentially private manner.
 We use the private gradient descent algorithm of~\cite{BST14}
 to provide private estimates of the 0.25, 0.75 predictions $(\ptf, \psf)$. This algorithm performs standard gradient descent, except that noise is added to a clipped version of the gradient at each round (clipped to range $[-\tau, \tau]^2$ for some setting of $\tau > 0$).
 \longversion{The number of calls to the gradient needs to be chosen carefully since we
 have to split our privacy budget amongst the gradient calls. 
 We note that we are yet to fully explore the full suite of parameter settings for this method. 
 See the Appendix
 for pseudo-code for our implementation and~\cite{BST14} for more information on 
 differentially private stochastic gradient descent. 
 We consider three versions of \texttt{DPGradDescent}: \texttt{DPGDPure}, \texttt{DPGDApprox}, and \texttt{DPGDzCDP}, which are $(\eps, 0)$-DP, $(\eps, \delta)$-DP, and
 $(\eps^2/2)$-zCDP algorithms, respectively. Zero-concentrated differential privacy (zCDP) is a slightly weaker notion of differential privacy that is well suited to iterative algorithms.
 For the purposes of this paper, we set $\delta = 2^{-30}$ whenever
 \texttt{DPGDApprox} is run, and set the privacy parameter of \texttt{DPGDzCDP} to allow for a fair comparison. \texttt{DPGDPure} provides the strongest privacy guarantee
 followed by \texttt{DPGDApprox} and then \texttt{DPGDzCDP}. As expected, \texttt{DPGDzCDP} typically has the best performance.}

\subsection{NoisyIntercept}

We also compare the above DP mechanisms to simply adding noise to the average $y$-value. For any given dataset
$(\bx, \by)\in([0, 1]\times [0,1])^n$, this method computes 
$\bar{y} = \frac{1}{n}\sum_{i=1}^ny_i$ and outputs a noisy estimate $\tilde{y} = \bar{y} + \Lap\left(0, \frac{1}{\eps n}\right)$ as the predicted $\tildeptf, \tildepsf$ estimates. This method performs well when the slope $\alpha$ is very small.

\subsection{A Note on Hyperparameters}
We leave the question of how to choose the optimal hyperparameters for each algorithm to future work. Unfortunately, since the optimal hyperparameter settings may reveal sensitive information about the dataset, one can not tune the hyperparameters on a hold-out set. However, we found that for most of the hyperparameters, once a good choice of hyperparameter setting was found, it could be used for a variety of similar datasets. Thus, one realistic way to tune the parameters may be to tune on a public dataset similar to the dataset of interest. For example, for applications using census data, one could tune the parameters on previous years' census data.

\longversion{We note that, as presented here, \texttt{NoisyStats} requires no hyperparameter tuning, while \texttt{DPExpTheilSen} and \texttt{DPSSTheilSen} only require a small amount of tuning (they require upper and lower bounds on the output). 
However, \texttt{NoisyStats} requires knowledge of the range of the inputs $x_i, y_i$, which is not required by the Theil-Sen methods.
Both \texttt{DPGradDescent} and \texttt{DPWideTheilSen} can be quite sensitive to the choice of hyperparameters. In fact, \texttt{DPExpTheilSen} is simply \texttt{DPWideTheilSen} with $\theta=0$, and we will often see a significant difference in the behaviour of these algorithms.
Preliminary experiments on the robustness of \texttt{DPWideTheilSen} to the choice of $\theta$ can be found in
the Appendix.}

\section{Experimental Outline}
\label{sec:exp}
The goal of this work is to provide insight and guidance into what features of a dataset should be considered when choosing a differentially private algorithm for simple linear regression. As such, in the following sections, we explore the behavior of these algorithms on a variety of real datasets. 
We also consider some synthetic datasets where we can further explore how properties of the dataset affect performance. 

\subsection{Description of the Data}

\subsubsection{Opportunity Insights data}

The first dataset we consider is a simulated version of the data used by the Opportunity Insights team in creating the Opportunity Atlas tool described in Section~\ref{OIusecase}. \longversion{
In the
appendix, we describe the data generation process used by the
Opportunity Insights team to generate the simulated data. } Each datapoint, $(x_i,y_i)$, corresponds to a pair, (parent income percentile rank, child income percentile rank)
\longversion{, where parent income is defined as the total pre-tax income at the household level, averaged over the years 1994-2000, and child income is defined as the total pre-tax income at the individual level, averaged over the years 2014-2015 when the children are between the ages of 31-37}. 
In the Census data, every state in the United States is partitioned into small neighborhood-level blocks called tracts. We perform the linear regression on each tract individually. 
The ``best'' differentially private algorithm differs from state to state, and even tract to tract. We focus on results for Illinois (IL)  which has a total of $n=219,594$ datapoints divided among 3,108 tracts.
 The individual datasets (corresponding to tracts in the Census data) each contain between $n=30$ and $n=400$ datapoints. The statistic $n\var(\bx)$ ranges between 0 and 25, with the majority of tracts having $n\var(\bx)\in[0,5]$. 
\longversion{We use $\eps=16$ in all our experiments on this data since that is the value approved by the Census for, and currently used in, the release of the Opportunity Atlas~\cite{CFHJP18}. }
\subsubsection{Washington, DC Bikeshare UCI Dataset}

Next, we consider a family of small datasets containing data from the Capital Bikeshare system, in Washington D.C., USA, from the years 2011 and 2012\footnote{This data is publicly available at \url{http://capitalbikeshare.com/system-data} \cite{Hadi:2013}.}. Each $(x_i,y_i)$ datapoint of the dataset contains the temperature ($x_i$) and user count of the bikeshare program ($y_i$) for a single hour in 2011 or 2012. 
The $x_i$ and $y_i$ values are both normalized so that they lie between 0 and 1 by a linear rescaling of the data.
\longversion{In order to obtain smaller datasets to work with, we segment this dataset into 288 (12x24) smaller datasets each corresponding to a (month, hour of the day) pair. Each smaller dataset}\shortversion{There are 288 datasets in this family, each of which } contains between $n=45$ and $n=62$ datapoints.
We linearly regress the number of active users of the bikeshare program on the temperature. While the variables are positively correlated, the correlation is not obviously linear --- in that the data is not necessarily well fit by a line ---
within many of these datasets, so we see reasonably large standard error values ranging between 0.0003 and 0.32 for this data.
Note that our privacy guarantee is at the level of hours rather than individuals so it serves more as an
example to evaluate the utility of our methods, rather than a real privacy application.
\longversion{While this is an important distinction to make for deployment of differential privacy, we will not dwell on it here since our goal to evaluate the statistical utility of our algorithms on real datasets. } \longversion{An example of one of these datasets is displayed in Figure~\ref{bikedata}.}

\longversion{
\begin{figure}
\includegraphics[width=0.4\textwidth]{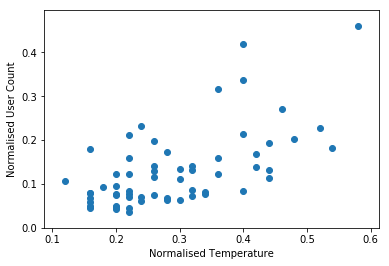}
\caption{Example of dataset from Washington, DC Bikeshare UCI Dataset.}\label{bikedata}
\end{figure}
}

\subsubsection{Carbon Nanotubes UCI Dataset}
While we are primarily concerned with evaluating the behavior of the private algorithms on small datasets, we also present results on a good candidate for private analysis: a large dataset with a strong linear relationship. In contrast to the Bikeshare UCI data and the Opportunity Insights data, this is a single dataset, rather than a family of datasets. This data is drawn from a dataset studying atomic coordinates in carbon nanotubes~\cite{Aci:2016}\longversion{. For each datapoint $(x_i,y_i)$, $x_i$ is the $u$-coordinate of the initial atomic coordinates of a molecule and $y_i$ is the $u$-coordinate of the calculated atomic coordinates after the energy of the system has been minimized. After we filtered out all points that are not contained in $[0,1]\times[0,1]$, the resulting dataset}\shortversion{, and } contains $n=10,683$ datapoints. \longversion{A graphical representation of the data is included in  Figure~\ref{carbondata}. }
Due to the size of this dataset, we run the efficient \texttt{DPTheilSenMatch} algorithm with $k=1$ on this dataset. This dataset does not contain sensitive information, however we have included it to evaluate the behaviour of the DP algorithms on a variety of real datasets.

\longversion{
\begin{figure}
\includegraphics[width=0.4\textwidth]{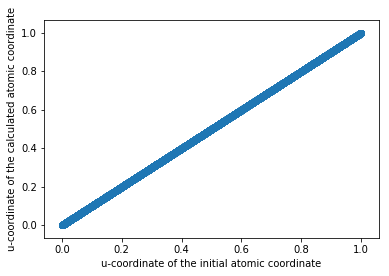}
\caption{Carbon Nanotubes UCI Dataset.
}\label{carbondata}
\end{figure}
}

\subsubsection{Stock Exchange UCI Dataset}

Our final real dataset studies the relationship between the Istanbul Stock Exchange and the USD Stock Exchange \cite{Akbilgic:2013}. \footnote{This dataset was collected from imkb.gov.tr and finance.yahoo.com.} It compares \{$x_i = $ Istanbul Stock Exchange national 100 index\} to \{$y_i = $ the USD International Securities Exchange\}. 
This dataset has $n=250$ datapoints\longversion{ and a representation is included in Figure~\ref{stockdata}}. This is a a smaller, noisier dataset than the Carbon Nanotubes UCI dataset.

\longversion{
\begin{figure}
\includegraphics[width=0.4\textwidth]{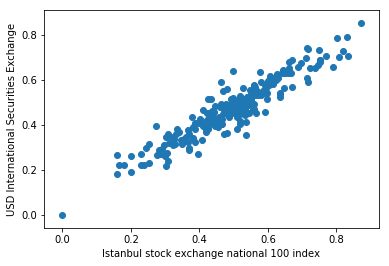}
\caption{The Stock Exchange UCI Dataset.}\label{stockdata}
\end{figure}
}

\subsubsection{Synthetic Data}
The synthetic datasets were constructed by sampling $x_i \in \mathbb{R}$, for $i = 1, \ldots, n$, independently
from a uniform distribution with $\bar{x} = 0.5$ and variance $\sigma_x^2$. For each $x_i$, the corresponding $y_i$ is generated as $y_i = \alpha x_i + \beta + e_i$, where $\alpha = 0.5, \beta = 0.2$, and $e_i$ is sampled from $\calN(0, \sigma_e^2)$. The $(x_i,y_i)$ datapoints are then clipped to the box $[0,1]^2$. The DP algorithms estimate the prediction at $\xnew$ using privacy parameter $\eps$.

The values of $n$, $\sigma_x^2$, $\sigma_e^2$, $\xnew$, and $\eps$ vary across the individual experiments.
The synthetic data experiments are designed to study which properties of the data and privacy regime determine the performance of the private algorithms. Thus, in these experiments, we vary one of parameters listed above
and observe the impact on the accuracy of the algorithms and their relative performance to each other. Since we know the ground truth on this synthetically generated data, we plot empirical error bounds that take into account both the sampling error and the error due to the DP algorithms, $\citrue(68)/\sigma(\hatpxnew)$. 
We evaluate \texttt{DPTheilSenkMatch} with $k=10$ in the synthetic experiments rather than \texttt{DPTheilSen}, since the former is computationally more efficient and still gives us insight into the performance of the latter.

\subsection{Hyperparameters}

\longversion{  \begin{table}
    \caption{Hyperparameters used in experiments on OI data and UCI datasets. 
    }
    \begin{tabular}{ccc} 
      \toprule
     Algorithms & OI/UCI data & Synthetic data\\ [0.5ex] 
      \midrule
      { \texttt{NoisyStats}} & None & None \\ 
      { \texttt{NoisyIntercept}} & None & N/A \\ 
      { \texttt{DPExpTheilSen}} & $r_l = -0.5, r_u = 1.5 $ & $r_l = -2, r_u = 2 $\\
      { \texttt{DPWideTheilSen}} & $r_l = -0.5, r_u = 1.5 $, & $r_l = -2, r_u = 2 $ \\
      &$\theta = 0.01$ & $\theta = 0.01$ \\ 
      { \texttt{DPSSTheilSen}} & $r_l = -0.5, r_u = 1.5 $, & $r_l = -2, r_u = 2 $\\
      &$d= 3$ & $d= 3$ \\ 
      { \texttt{DPGradDescent}} & 
    $\tau = 1$,
      T = 80 & $\tau = 1$,
  T = 80\\
      \bottomrule
    \end{tabular}
    \label{tab:OI-hyperparams}
  \end{table}}

Hyperparameters were tuned on the semi-synthetic Opportunity Insights data by experimenting with many different choices, and choosing the best. The hyperparameters are listed in Table~\ref{tab:OI-hyperparams}.
We leave the question of optimizing hyperparameters
in a privacy-preserving way to future work. 
Hyperparameters for the UCI and synthetic datasets were chosen according to choices that seemed to perform well on the Opportunity Insights datasets. This ensures that the hyperparameters were not tuned to the specific datasets (ensuring validity of the privacy guarantees), but also leaves some room for improvement in the performance by more careful hyperparameter tuning.

\section{Results and Discussion}\label{discussion}

In this section we discuss the findings from the experiments described in the previous section. In the majority of the real datasets tested, \texttt{DPExpTheilSen} is the best performing algorithm. We'll discuss some reasons why \texttt{DPExpTheilSen} performs well, as well as some regimes when other algorithms perform better.

A note on the statement of privacy parameters in the experiments. We will state the privacy budget used to compute the pair $(\ptf, \psf)$, however we will only show empirical error bounds for $\ptf$. The empirical error bounds for $\psf$ display similar phenomena. The algorithms \texttt{NoisyStats} and \texttt{DPGradDescent} inherently release both point estimates together so the privacy loss is the same whether we release the pair $(\ptf, \psf)$ or just $\ptf$. However, \texttt{DPExpTheilSen}, \texttt{DPExpWideTheilSen}, \texttt{DPSSTheilSen} and the algorithm used by the Opportunity Insights team use half their budget to release $\ptf$ and half their budget to release $\psf$, separately.

\subsection{Real Data and the Opportunity Insights Application}

Figure~\ref{OIdata} shows the results of all the DP linear regression algorithms, as well as the mechanism used by the Opportunity Insights team, on the Opportunity Insights data for the state\longversion{s } of Illinois \longversion{and North Carolina. } For each algorithm, we build an empirical cumulative density distribution of the empirical error bounds set over the tracts in that state. The vertical dotted line in Figures~\ref{OIdata} intercepts each curve at the point where the \emph{noise due to privacy exceeds the standard error}.
The Opportunity Insights team used a heuristic method inspired by, but not satisfying, DP to deploy the Opportunity Altas \cite{CFHJP18}. Their privacy parameter of $\eps=16$
was selected by the
Opportunity Insights team and a Census Disclosure Review Board by balancing the privacy and utility considerations. Figure~\ref{OIdata} shows that there exist differentially private algorithms that are competitive with, and in many cases more accurate than, the algorithm currently deployed in the Opportunity Atlas. 

Additionally, while the methods used by the Opportunity Insights team are highly tailored to their setting
relying on coordination across Census tracts, in order to compute an upper bound on the local sensitivity, as discussed in Section~\ref{robustnessalgodesign}, the formally differentially private methods are general-purpose and do not require coordination across tracts.

\longversion{
\begin{figure}
\begin{subfigure}{0.5\textwidth}
\includegraphics[scale=0.6]{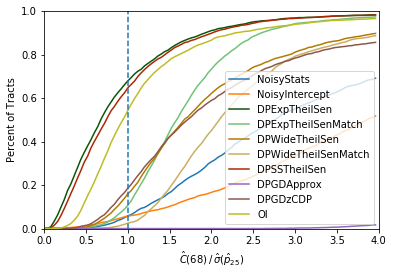}
    \caption{Illinois (IL)}
    \label{fig:IL_CDF}
    \end{subfigure}
    \begin{subfigure}{0.5\textwidth}
\includegraphics[width=\textwidth]{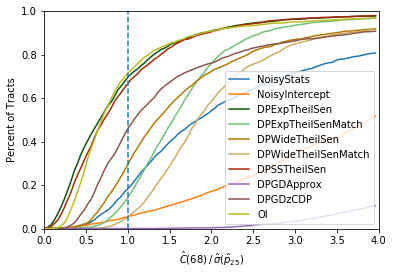}
    \caption{North Carolina (NC)}
    \label{fig:NC_CDF}
    \end{subfigure}
  \caption{Empirical CDF for the Empirical 68\% error bounds, $\ci$, normalized by empirical OLS standard error when evaluated on Opportunity Insights data for the states of Illinois and North Carolina. The algorithm denoted by OI refers to the heuristic algorithm used in the deployment of the Opportunity Altas Tool~\cite{ChettyF19}. {Privacy parameter $\eps=16$ for the pair $(\ptf, \psf)$.}}
  \label{OIdata}
\end{figure}}

Figures~\ref{bike} shows empirical cumulative density distributions of the empirical error bounds on the 288 datasets in the Bikeshare dataset. Figure~\ref{stock} shows the empirical cumulative density function of the output distribution on the Stockexchange dataset. Note that this is a different form to the empirical CDFs which appear in Figure~\ref{OIdata} and Figure~\ref{bike}. Figure~\ref{carbon} shows the empirical cumulative density function of the output distribution on the Carbon Nanotubes dataset. Figures~\ref{bike}, \ref{stock} and~\ref{carbon} show that on the three other real datasets, for a range of realistic $\eps$ values the additional noise due to privacy, in particular using \texttt{DPExpTheilSen}, is less than the standard error. 

\begin{figure}
\includegraphics[width=0.5\textwidth]{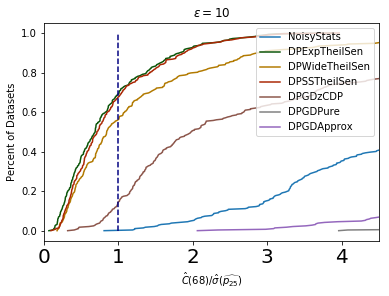}
\caption{Bikeshare data. Empirical CDFs of the performance over the set of datasets when $\eps=10$. 
}\label{bike}
\end{figure}

\begin{figure}
\includegraphics[width=0.5\textwidth]{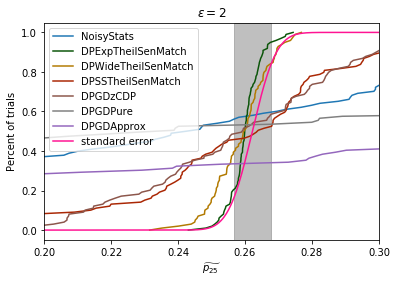}
\caption{
Stock Exchange UCI Data. Empirical cdf of the output distribution of the estimate of $\ptf$ after 100 trials of each algorithm with $\eps=2$. The grey region includes all the values that are within one standard error of $\hatsigma(\hatptf)$. The curve labelled ``standard error`` shows the non-private posterior belief on the value of the $\ptf$ assuming Gaussian noise.}\label{stock}
\end{figure}

\begin{figure}
        \centering
        \begin{subfigure}[b]{0.5\textwidth}
            \centering
            \includegraphics[width=\textwidth]{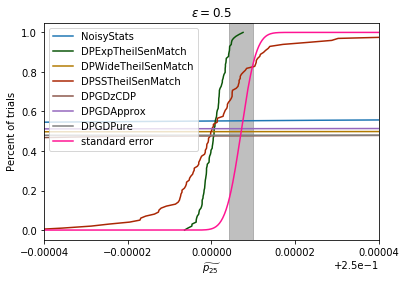}
            \vspace*{-7mm}
            \caption{Domain knowledge that $\ptf,\psf~\in~[-0.5, 1.5]$}
            \label{carbon}
        \end{subfigure}
        \begin{subfigure}[b]{0.5\textwidth}   
            \centering 
            \includegraphics[width=\textwidth]{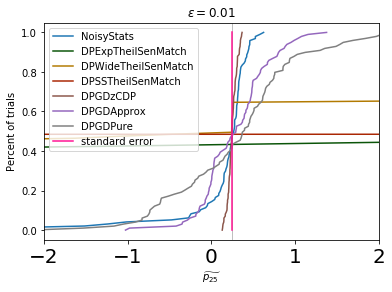}
            \vspace*{-7mm}
            \caption{Domain knowledge that $\ptf,\psf~\in~[-50, 50]$}    
            \label{carbonNSwins}
        \end{subfigure}
        \caption{Carbon nanotubes UCI Data. Empirical cdf of the output distribution of the estimate of $\ptf$ after 100 trials of each algorithm. The grey region includes all the values that are within one standard error of $\hatptf$. The curve labelled ``standard error`` shows the non-private posterior belief on the value of the $\ptf$ assuming Gaussian noise.} 
\end{figure}

\subsection{Robustness vs. Non-robustness: Guidance for Algorithm Selection}\label{restriction}

The DP algorithms we evaluate can be divided into two classes, \emph{robust} DP estimators based on Theil-Sen --- \texttt{DPSSTheilSen}, \texttt{DPExpTheilSen} and \texttt{DPWideTheilSen} --- and \emph{non-robust} DP estimators based on OLS --- \texttt{NoisyStats} and \texttt{GradDescent}. Experimentally, we found that the algorithm's behaviour tend to be clustered in these two classes, with the robust estimators outperforming the non-robust estimators in a wide variety of parameter regimes.
In most of the experiments we saw in the previous section (Figures~\ref{OIdata},~\ref{bike}, ~\ref{stock} and, ~\ref{carbon}), \texttt{DPExpTheilSen} was the best performing algorithm, followed by \texttt{DPWideTheilSen} and \texttt{DPSSTheilSen}. 
However, below we will see that in experiments on synthetic data, we found that the non-robust estimators outperform the robust estimators in some parameter regimes. 
\subsubsection{\texttt{NoisyStats} and \texttt{DPExpTheilSen}}

In Figure~\ref{fig:synth-algs-ground-truth}, we investigate the relative performance ($\citrue(68) / \sigma(\hatptf)$) of the algorithms in several parameter regimes of $n, \eps$, and $\sigma_x^2$ on synthetically generated data. 
 For each parameter setting, for each algorithm, we plot of the average value of $\citrue(68) / \sigma(\hatptf)$ over $50$ trials on a single dataset, and average again over $500$ independently sampled datasets. 
Across large ranges for each of these three parameters ($n\in[30, 10,000]$; $\sigma_x^2 \in[0.003, 0.03]$; $\eps\in [0.01, 10]$, all varied on a logarithmic scale), we see that \longversion{either } \texttt{DPExpTheilSen}\longversion{,~\texttt{DPGDzCDP}, } or \texttt{NoisyStats} is consistently the best performing algorithm---or close to it. \longversion{ Note that of these three algorithms, only \texttt{DPExpTheilSen} and \texttt{NoisyStats} fulfill the stronger pure $(\eps, 0)$-DP privacy guarantee. } We see that \longversion{\texttt{DPGDzCDP} and } \texttt{NoisyStats} trend\shortversion{s } towards taking over as the best algorithm\longversion{s } as $\eps n \sigma_x^2$ increases.

\begin{figure}
        \centering
        \begin{subfigure}[b]{0.5\textwidth}
            \centering
            \includegraphics[width=\textwidth]{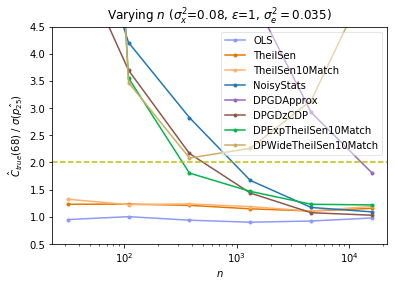}
            \vspace*{-7mm}
            \caption{{\small Varying n}
            }
            \label{fig:synth-n-ground-truth}
        \end{subfigure}
        \begin{subfigure}[b]{0.5\textwidth}   
            \centering 
            \includegraphics[width=\textwidth]{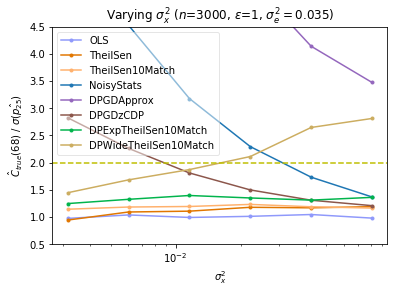}
            \vspace*{-7mm}
            \caption{{\small Varying $\sigma_x^2$ }}    
            \label{fig:synth-varx-ground-truth}
        \end{subfigure}
        \begin{subfigure}[b]{0.5\textwidth}   
            \centering 
            \includegraphics[width=\textwidth]{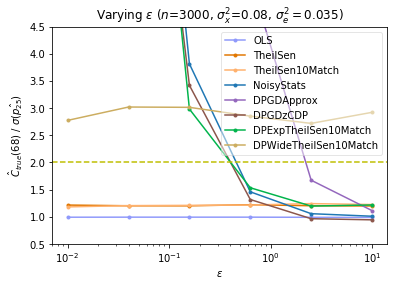}
            \vspace*{-7mm}
            \caption{{\small Varying $\eps$ 
            }}    
            \label{fig:synth-eps-ground-truth}
        \end{subfigure}
        \caption{ Relative error ($\citrue / \sigma(\hatptf)$) of DP and non-private algorithms on synthetic data..
}
\label{fig:synth-algs-ground-truth}
\end{figure}

    \begin{figure}
        \begin{subfigure}[b]{0.23\textwidth}
            \centering
            \includegraphics[width=\textwidth]{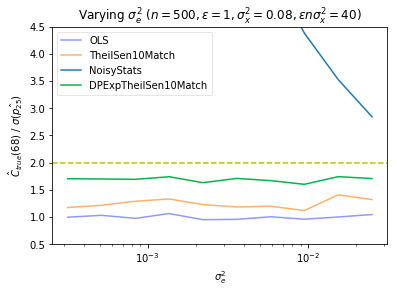}
            \vspace*{-7mm}
            \captionsetup{width=0.9\textwidth}
            \caption{{\small $\eps n \sigma_x^2 = 40$; varying $\sigma_e^2$}}
            \label{fig:synth-epsnvarx-small-vare-ground-truth}
        \end{subfigure}
        \begin{subfigure}[b]{0.23\textwidth}
            \centering
            \includegraphics[width=\textwidth]{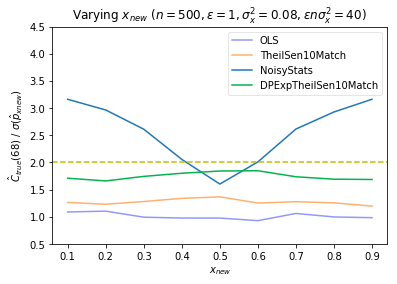}
            \vspace*{-7mm}
            \captionsetup{width=0.9\textwidth}
            \caption{{\small $\eps n \sigma_x^2 = 40$; varying $\xnew$}}
            \label{fig:synth-epsnvarx-small-xnew-ground-truth}
        \end{subfigure}
        \begin{subfigure}[b]{0.23\textwidth}
            \centering
            \includegraphics[width=\textwidth]{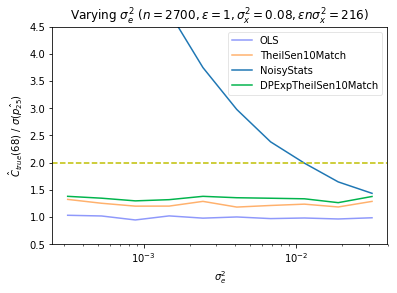}
            \vspace*{-7mm}
            \captionsetup{width=0.9\textwidth}
            \caption{{\small $\eps n \sigma_x^2 = 216$; varying $\sigma_e^2$}}
            \label{fig:synth-epsnvarx-medium-vare-ground-truth}
        \end{subfigure}
        \begin{subfigure}[b]{0.23\textwidth}
            \centering
            \includegraphics[width=\textwidth]{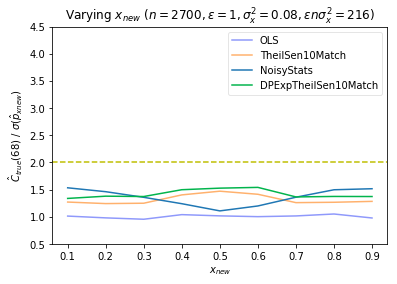}
            \vspace*{-7mm}
            \captionsetup{width=0.9\textwidth}
            \caption{{\small $\eps n \sigma_x^2 = 216$; varying $\xnew$}}
            \label{fig:synth-epsnvarx-medium-xnew-ground-truth}
        \end{subfigure}
        \begin{subfigure}[b]{0.23\textwidth}
            \centering
            \includegraphics[width=\textwidth]{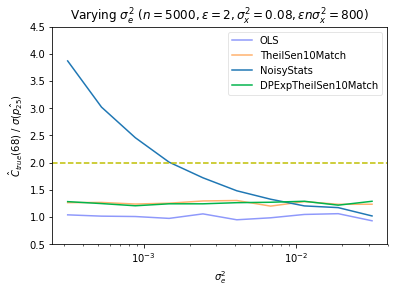}
            \vspace*{-7mm}
            \captionsetup{width=0.9\textwidth}
            \caption{{\small $\eps n \sigma_x^2 = 800$; varying $\sigma_e^2$}}
            \label{fig:synth-epsnvarx-large-vare-ground-truth}
        \end{subfigure}
        \begin{subfigure}[b]{0.23\textwidth}
            \centering
            \includegraphics[width=\textwidth]{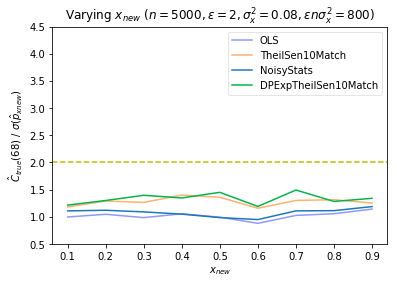}
            \vspace*{-7mm}
            \captionsetup{width=0.9\textwidth}
            \caption{{\small $\eps n \sigma_x^2 = 800$; varying $\xnew$}
            }
            \label{fig:synth-epsnvarx-large-xnew-ground-truth}
        \end{subfigure}
        \caption{{ Comparing \texttt{NoisyStats} and \texttt{DPExpTS10Match} on synthetic data as $\sigma_e^2$ and $\xnew$ vary, for $\bar{x} = 0.5$. Plotting $\citrue / \sigma(\hatptf)$. OLS and TheilSen10Match included for reference. 
        }} 
        \label{fig:synth-ns-ts-ground-truth}
    \end{figure}

For parameter regimes in which the non-robust algorithms outperform the robust estimators, \texttt{NoisyStats} is preferable \longversion{to \texttt{DPGDzCDP} } since it is more efficient, requires no hyperparameter tuning (except bounds on the inputs $x_i$ and $y_i$), fulfills a stronger privacy guarantee, and releases the noisy sufficient statistics with no additional cost in privacy.
Experimentally, we find that the main indicator for deciding between robust estimators and non-robust estimators is the quantity $\eps n\var(\bx)$ (which is a proxy for $\eps n \sigma_x^2$). 
Roughly, when $\eps n\var(\bx)$ and $\sigma_e^2$ 
are both large, \texttt{NoisyStats} is \longversion{close to } optimal among the DP algorithms tested; otherwise, the robust estimator \texttt{DPExpTheilSen} typically has lower error. 
Hyperparameter tuning and the quantity $|\xnew-\bar{x}|$ also play minor roles in determining the relative performance of the algorithms.

\subsubsection{Role of $\eps n\var(\bx)$}

The quantity $\eps n\var(\bx)$ is related to the size of the dataset, how concentrated the independent variable of the data is and how private the mechanism is. It appears to be a better indicator of the performance of DP mechanisms than any of the individual statistics $\eps, n$, or $\var(\bx)$ in isolation. In Figure~\ref{fig:synth-ns-ts-ground-truth} we compare the performance of \texttt{NoisyStats} and \texttt{DPExpTheilSen10Match} as we vary $\sigma_e^2$ and $\xnew$. For each parameter setting, for each algorithm, the error is computed as the average error over $20$ trials and $500$ independently sampled datasets.  Notice that the blue line presents the error of the non-private OLS estimator, which is our baseline. The quantity we control in these experiments is $\eps n \sigma_x^2$, although we expect this to align closely with $\eps n\var(\bx)$, which is the empirically measurable quantity.
In all of our synthetic data experiments, in which \longversion{the $x_i$'s are uniform and } the $e_i$'s are Gaussian, once $\eps n \sigma_x^2>400$ and $\sigma_e^2\ge 10^{-2}$, \texttt{NoisyStats} is \longversion{close to } the best performing algorithm.
 It is also important to note that once $\eps n\var(\bx)$ is large, both \texttt{NoisyStats} and \texttt{DPExpTheilSen} perform well. See Figure~\ref{fig:synth-ns-ts-ground-truth} for a demonstration of this on the synthetic data.

The error, as measured by $\citrue(68)$, of both OLS and Theil-Sen estimators converges to 0 as $n\to\infty$ at the same asymptotic rate. However, OLS converges a constant factor faster than Theil-Sen resulting in its superior performance on relatively large datasets.
As $\eps n\var(\bx)$ increases, the error of \texttt{NoisyStats} decreases, and the outputs of this algorithm concentrate around the OLS estimate. While the outputs of \texttt{DPExpTheilSen} tend to be more concentrated, they converge to the Theil-Sen estimate, which has higher sampling error. Thus, as we increase $\eps n\var(\bx)$, eventually the additional noise due to privacy added by \texttt{NoisyStats} is less than the difference between the OLS and Theil-Sen point estimates, resulting in \texttt{NoisyStats} outperforming \texttt{DPExpTheilSen}. This phenomenon can be seen in Figure~\ref{fig:synth-algs-ground-truth} and Figure~\ref{fig:synth-ns-ts-ground-truth}.

\longversion{The classifying power of $\eps n\var(\bx)$ is a result of its impact on the performance of \texttt{NoisyStats}.
Recall that \texttt{NoisyStats} works by using noisy sufficient statistics within the closed form solution for OLS given in Equation~\ref{closedform}. The noisy sufficient statistic $\nvar(\bx)+\Lap(0, 3\Delta_2/\eps)$, appears in the denominator of this closed form solution. When $\eps n\var(\bx)$ is small, this noisy sufficient statistic has constant mass concentrated near, or below, 0, and hence the inverse, $1/(\nvar(\bx)+\Lap(0, 3\Delta_2/\eps))$, which appears in the \texttt{NoisyStats}, can be an arbitrarily bad estimate of $1/\nvar(\bx)$. In contrast, when $\eps n\var(\bx)$ is large, the distribution of $\nvar(\bx)+\Lap(0, 3\Delta_2/\eps)$ is more concentrated around $\nvar(\bx)$, so that with high probability $1/(\nvar(\bx)+\Lap(0, 3\Delta_2/\eps))$ is close to $1/\nvar(\bx)$. 
}

In Figures~\ref{fig:synth-n-ground-truth}, \ref{fig:synth-varx-ground-truth} and ~\ref{fig:synth-eps-ground-truth}, the performance of all the algorithms, including \texttt{NoisyStats} and \texttt{DPExpTheilSen}, are shown as we vary each of the parameters $\eps, n$ and $\sigma_x^2$, while holding the other variables constant, on synthetic data. In doing so, each of these plots is effectively varying $\eps n\var(\bx)$ as a whole. These plots suggest that all three parameters help determine which algorithm is the most accurate. Figure~\ref{fig:synth-ns-ts-ground-truth} confirms that $\eps n\var(\bx)$ is a strong indicator of the relative performance of \texttt{NoisyStats} and \texttt{DPExpTheilSen}, even as other variables in the OLS standard error equation (\ref{standarderror}) -- including the variance of the noise of the dependent variable, $\sigma_e$, and the difference between $\xnew$ and the mean of the $x$ values, $|\xnew - \bar{x}|$ -- are varied.

\subsubsection{The Role of $\sigma_e^2$}
One main advantage that all the \texttt{DPTheilSen} algorithms have over \texttt{NoisyStats} is that they are better able to adapt to the specific dataset. 
When $\sigma_e^2$ is small, there is a strong linear signal in the data that the non-private OLS estimator and the \texttt{DPTheilSen} algorithms can leverage. Since the noise addition mechanism of \texttt{NoisyStats} does not leverage this strong signal, its relative performance is worse when $\sigma_e^2$ is small. Thus, $\sigma_e^2$ affects the relative performance of the algorithms, even when $\eps n\var(\bx)$ is large. 
We saw this effect in Figure~\ref{carbon} on the Carbon Nanotubes UCI data, where $\eps n\var(x)=889.222$ is large, but \texttt{NoisyStats} performed poorly relative to the \texttt{DPTheilSen} algorithms. 

In each of the plots \ref{fig:synth-epsnvarx-small-vare-ground-truth},  \ref{fig:synth-epsnvarx-medium-vare-ground-truth}, and \ref{fig:synth-epsnvarx-large-vare-ground-truth}, the quantity $\eps n\var(\bx)$ is held constant while $\sigma_e^2$ is varied. These plots confirm that the performance of \texttt{NoisyStats} degrades, relative to other algorithms, when $\sigma_e^2$ is small. As $\sigma_e^2$ increases, the relative performance of \texttt{NoisyStats} improves until it drops below the relative performance of the TheilSen estimates. The cross-over point varies with $\eps n\var(\bx)$. In Figure~\ref{fig:synth-epsnvarx-small-vare-ground-truth}, we see that when $\eps n\var(\bx)$ is small, the methods based on Theil-Sen are the better choice for all values of $\sigma_e^2$ studied. When we increase $\eps n\var(\bx)$ in Figure~\ref{fig:synth-epsnvarx-large-vare-ground-truth} we see a clear cross over point. 
These plots add evidence to our conjecture that robust estimators are a good choice when $\eps  n\var(\bx)$ and $\sigma_e^2$ are both small, while non-robust estimators perform well when $\eps n \var\bx$ is large.

\subsubsection{The Role of $|\xnew-\bar{x}|$}

The final factor we found that plays a role in the relative performance of \texttt{NoisyStats} and \texttt{DPExpTheilSen} is $|\xnew-\bar{x}|$. This effect is less pronounced than that of $\eps n\var(\bx)$ or $\sigma_e^2$, and seems to rarely change the ordering of DP algorithms. In Figures~\ref{fig:synth-epsnvarx-small-xnew-ground-truth}, \ref{fig:synth-epsnvarx-medium-xnew-ground-truth} and \ref{fig:synth-epsnvarx-large-xnew-ground-truth}, we explore the effect of this quantity on the relative performance of OLS, TheilSen10Match, \texttt{DPExpTheilSen10Match}, and \texttt{NoisyStats}.
We saw in Equation~\ref{standarderror} that $|\xnew-\bar{x}|$ affects the standard error $\hat\sigma(\hatptf)$ - the further $\xnew$ is from the centre of the data, $\bar{x}$, the less certain we are about the OLS point estimate. All algorithms have better performance when $|\xnew-\bar{x}|$ is small; however, this effect is most pronounced with \texttt{NoisyStats}. In \texttt{NoisyStats}, $\tildepxnew = \tilde{\alpha}(\xnew-\bar{x})+\bar{y}+L_3$ where
$L_3\sim\Lap(0, 3(1+|\tilde\alpha|)/(\eps n))$ and $\tilde{\alpha}$ and $\tilde{\beta}$ are the noisy slope and intercepts respectively. Thus, we expect a large $|\xnew-\bar{x}|$ to amplify the error present in $\tilde{\alpha}$. 
We
vary $\xnew$ between 0 and 1. As expected, the error is minimized when $\xnew = 0.5$,
since $\bar{x} = 0.5$.
Since we expect the error in $\tilde{\alpha}$ to decrease as we increase $\eps n\var(\bx)$, this explains why the quantity $|\xnew-\bar{x}|$ has a larger effect when $\eps n\var(\bx)$ is small.

\subsubsection{The Role of Hyperparameter Tuning}

A final major distinguishing feature between the \texttt{DPTheilSen} algorithms, \texttt{NoisyStats} and \texttt{DPGradDescent} is the amount of prior knowledge needed by the data analyst to choose the hyperparameters appropriately. Notably, \texttt{NoisyStats} does not require any hyperparameter tuning other than a bound on the data. The \texttt{DPTheilSen} algorithms require some reasonable knowledge of the range that $\ptf$ and $\psf$ lie in\longversion{ in order to set $r_l$ and $r_u$}. Finally, \texttt{DPGradDescent} requires some knowledge of where the input values lie, so it can set $\tau$ and $T$. 

In Figure~\ref{carbonNSwins}, \texttt{NoisyStats} and all three \texttt{DPGradDescent} algorithms outperform the robust estimators.
This experiment differs from Figure~\ref{carbon} in two important ways: the privacy parameter $\eps$ has decreased from 1 to 0.01, and the feasible output region for the DP TheilSen methods has increased from $[-0.5,1.5]$ to $[-50,50]$. 
When $\eps$ is large, the \texttt{DPTheilSen} algorithms are robust to the choice of this region since any area outside the range of the data is exponentially down weighted (see Figure~\ref{diag:em}). However, when $\eps$ is small, the size of this region can have a large effect on the stability of the output. As $\eps$ decreases, the output distributions of the \texttt{DPTheilSen} estimators are flattened, so that they are essentially sampling from the uniform distribution on the range of the parameter. 
This effect likely explains the poor performance of the robust estimators in Figure~\ref{carbonNSwins}, and highlights the importance of choosing hyperparameters carefully. If $\eps$ is small ($\eps$ much less than 1) and the analyst does not have a decent estimate of the range of $\ptf$, then $\texttt{NoisyStats}$ may be a safer choice than \texttt{DPTheilSen}.

\subsection{Which robust estimator?}

In the majority of the regimes we have tested, \texttt{DPExpTheilSen} outperforms all the other private algorithms. While $\texttt{DPSSTheilSen}$ can be competitive with \texttt{DPExpTheilSen} and \texttt{DPWideTheilSen}, it rarely seems to outperform them. However, \texttt{DPWideTheilSen} can significantly outperform \texttt{DPExpTheilSen} when the standard error is small.
Figure~\ref{bikeTS} compares the performance of {DPExpTheilSen} and \texttt{DPWideTheilSen} on the Bikeshare UCI data. When there is little noise in the data we expect the set of pairwise estimates that Theil-Sen takes the median of to be highly concentrated. We discussed in Section~\ref{alg:dpts} why this is a difficult setting for \texttt{DPExpTheilSen}: in the continuous setting, the exponential mechanism based median algorithm can fail to put sufficient probability mass around the median, even if the data is concentrated at the median (see Figure~\ref{diag:em}). 
\texttt{DPWideTheilSen} was designed exactly as a fix to this problem. The parameter $\theta$ needs to be chosen carefully. In Figure~\ref{carbon}, $\theta$ is set to be much larger than the standard error, resulting in \texttt{DPWideTheilSen} performing poorly. 

\begin{figure}[!hbt]
\includegraphics[width=0.45\textwidth]{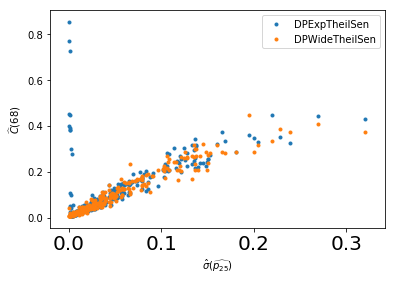}
\caption{
Comparison of \texttt{DPExpTheilSen} and \texttt{DPWideExpTheilSen} on Bikeshare UCI data with $\eps=2$ and $\theta=0.01$. Datasets are sorted by their standard errors. For each dataset, there is a dot corresponding to the $\hat{C}(68)$ value of each algorithm. Both dots lie on the same vertical line.
}
\label{bikeTS}
\end{figure}

\longversion{
\subsection{Which non-robust estimator?}

There are two main non-robust estimators we consider: \texttt{DPGradDescent} and
\texttt{NoisyStats}. \texttt{DPGradDescent} has three versions --
\texttt{DPGDPure}, \texttt{DPGDApprox}, and \texttt{DPGDzCDP} -- corresponding to the
pure, approximate, and zero-concentrated variants. Amongst the \texttt{DPGradDescent} algorithms,
as expected, \texttt{DPGDzCDP} provides the best utility followed by \texttt{DPGDApprox},
and then \texttt{DPGDPure}. But how do these compare to \texttt{NoisyStats}?
\texttt{NoisyStats} outperforms both \texttt{DPGDPure} and \texttt{DPGDApprox} for small $\delta$ (e.g. $\delta=2^{-30}$ in our experiments).
\texttt{DPGDzCDP} consistently outperforms \texttt{NoisyStats}, but the gap in performance is small in the regime where the non-robust estimators
beat the robust estimators.  Moreover, \texttt{NoisyStats} achieves a stronger privacy guarantee (pure $(\eps, 0)$-DP rather than $\eps^2/2$-zCDP).  A fairer comparison is to use the natural $\eps^2/2$-zCDP analogue of \texttt{NoisyStats} (using Gaussian noise and zCDP composition), in which case we have found that the advantage of \texttt{DPGDzCDP} significantly shrinks and in some cases is reversed. (Experiments omitted.)

The performance of the \texttt{DPGradDescent}
algorithms also depend on hyperparameters that need to be carefully tuned,
such as the number of gradient calls $T$ and the clip range $[-\tau, \tau]$. Since \texttt{NoisyStats} requires less hyperparameters, this
makes \texttt{DPGradDescent} potentially harder to use in practice.
In addition, \texttt{NoisyStats} is more efficient and
can be used to release the noisy sufficient statistics with no additional privacy loss.  Since the performance of the two algorithms is similar in the regime where non-robust methods appear to have more utility than
the robust ones, the additional benefits of \texttt{NoisyStats} may make it preferable in practice. 

We leave a more thorough
evaluation and optimization of these algorithms in the regime of large $n$, including how to optimize the
hyperparameters in a DP manner, to future work.

}

\longversion{\subsection{Analyzing the bias}

Let $(\ptf^{TS}, \psf^{TS})$ be the prediction estimates produced using the non-private TheilSen estimator. 
The non-robust DP methods -- \texttt{NoisyStats} and \texttt{DPGradDescent} -- approach $(\hatptf, \hatpsf)$ as $\eps\to\infty$, 
while the \texttt{DPTheilSen} methods approach $(\ptf^{TS}, \psf^{TS})$ as $\eps\to\infty$.
For any fixed dataset, $(\hatptf, \hatpsf)$ and
$(\ptf^{TS}, \psf^{TS})$ are not necessarily equal.  
A good representation of this bias can be seen in Figure~\ref{carbon}. However, 
both the TheilSen estimator and the OLS estimator are consistent unbiased estimators of the true slope in simple linear regression. That is, as $n\to\infty$, both $(\ptf^{TS}, \psf^{TS})$ and $(\hatptf, \hatpsf)$ tend to the true value $(\ptf, \psf)$. Thus, all the private algorithms output the true prediction estimates as $n\to\infty$, for a fixed $\eps$.}

\section{Conclusion}

It is possible to design DP simple linear regression algorithms where the distortion added by the private algorithm is less than the standard error, even for small datasets. In this work, we found that in order to achieve this we needed to switch from OLS regression to the more robust linear regression estimator, Theil-Sen. We identified key factors that analysts should consider when deciding whether DP methods based on robust or non-robust estimators are right for their application.

\begin{acks}
We thank Raj Chetty, John Friedman, and Daniel Reuter from Opportunity Insights for many motivating discussions, providing us with simulated Opportunity Atlas data for our experiments, and helping us implement the OI algorithm.  We also thank Cynthia Dwork and James Honaker for helpful conversations at an early stage of this work; Garett Bernstein and Dan Sheldon for discussions about their work on private Bayesian inference for linear regression; Sofya Raskhodnikova, Thomas Steinke, and others at the Simons Institute Spring 2019 program on data privacy for helpful discussions on specific technical aspects of this work. 
\end{acks}

\bibliographystyle{ACM-Reference-Format}
\bibliography{main2}
\appendix
\label{pseudocode}

\section{Opportunity Insights Application}

\subsection{Data Generating Process for OI Synthetic Datasets}

In this subsection, we describe the data generating process for simulating census microdata from the 
Opportunity Insights team.
In the model, assume that each set of parents $i$ in tract $t$ and state $m$ has one child
(also indexed by $i$).

\textbf{Size of Tract}: The size of each tract in any state is a random variable. 
Let $\Exp(b)$ represent the exponential distribution with scale $b$. Then if $n_{tm}$ is the number of people
in tract $t$ and state $m$, then $n_{tm}\sim\lfloor\Exp(52) + 20\rfloor$. This distribution over 
simulated counts was chosen by the Opportunity Insights
team because it closely matches the distribution of tract sizes in the
real data.

\textbf{Linear Income Relationship}: Let $x_{itm}$ be the child income given the parent income $y_{itm}$, then we enforce the following relationship between $x_{itm}$ and $y_{itm}$:
$$
\ln(x_{itm}) = \alpha_{tm} + \beta_{tm}\ln(y_{itm}) + e_i,
$$
where $e_i\sim\calN(0, (0.20)^2)$ and $\alpha_{tm}, \beta_{tm}$ are calculated from public estimates of
child income rank given the parent income rank.
\footnote{
Gleaned from some tables the Opportunity Insights team publicly released with tract-level (micro) data.
See \url{https://opportunityatlas.org/}.
The dataset used to calculate the $\alpha_{tm}, \beta_{tm}$
is aggregated from some publicly-available income data
for all tracts within all U.S. states between 1978 and 1983.
}
Next, $p_{im}$ is calculated as the parent $i$'s percentile income within the state $m$'s parent
income distribution (and rounded up to the 2nd decimal place).

\textbf{Parent Income Distribution}: Let $\mu_{tm}$ denote the public estimate of the mean 
household income for tract $t$ in state $m$ (also obtained from publicly-available income data
used to calculate $\alpha_{tm}, \beta_{tm}$).
Empirically, the Opportunity Insights team found that the within-tract
standard deviation of parent incomes is about twice the between-tract standard deviation. 
Let $\Var(\mu_{tm})$ denote the sample variance of the estimator $\mu_{tm}$. Then 
enforce that
$\Var_{tm}(y_{itm}) = 4\Var(\mu_{tm})$
where
$\Var_{tm}(y_{itm})$ is the variance of parental income $y_{itm}$ in tract $t$ and state $m$.
Furthermore, assume that $y_{itm}$ are lognormally distributed within each tract and draw
$\ln(y_{itm})$ from \\ $\calN(\E_{tm}[\ln(y_{itm})], \Var_{tm}[\ln(y_{itm})])$ for $i=1, \ldots, n_{tm}$, where $$\E_{tm}[\ln(y_{itm})] = 2\ln(\mu_{tm}) - 0.5\ln(\Var_{tm}(y_{itm}) + \mu_{tm}^2),$$
and $$\Var_{tm}[\ln(y_{itm})] = -2\ln(\mu_{tm}) + \ln(\Var_{tm}(y_{itm}) + \mu_{tm}^2).$$

\subsection{The Maximum Observed Sensitivity Algorithm}
\label{sec:mos}

Opportunity Insights~\cite{ChettyF19} 
provided a practical method -- which they term the
``Maximum Observed Sensitivity'' (MOS) algorithm -- to reduce the privacy loss
of their released estimates. This method is not formally
differentially private. We use MOS or OI interchangeably to refer to their
statistical disclosure limitation method.
The crux of their algorithm is as follows:
The maximum observed sensitivity, corresponding to an 
upper envelope on
the largest local sensitivity across tracts in a state in the dataset,
is calculated and Laplace noise
of this magnitude divided by the number of people in that cell is added to the
estimate and then released. The statistics they release include 0.25, 0.75
percentiles per cell, the standard error of these percentiles, and the count.

Two notes are in order. First, the MOS algorithm does not calculate the local sensitivity
exactly but uses a lower bound by overlaying an $11\times 11$ grid on the $[0, 1]\times[0, 1]$
space of possible $\bx, \by$ values. Then the algorithm proceeds to add a datapoint from this grid to the dataset and
calculate the maximum change in the statistic, which is then used as a lower bound.
Second, analysis are performed only on census tracts that satisfy the following property:
they must have at least 20 individuals with 10\% of parent income percentiles in that tract above the state parent income median
percentile and 10\% below. If a tract does not satisfy this condition then no regression estimate is released for that tract.

\section{Some Results in Differential Privacy}

In this section we will briefly review some of the fundamental definitions and results pertaining to general differentially private algorithms.

For any query function $f:\calX^n\rightarrow\reals^K$ let \[\gs_f = \max_{d\sim d'} \norm{f(d)-f(d')},\] called the global sensitivity, be the maximum amount the query can differ on neighboring datasets.

\begin{theorem}[Laplace Mechanism~\cite{DMNS06}]
For any privacy parameter $\eps > 0$ and any given query function
$f:\calX^n\rightarrow\reals^K$ and database
$d \in\calX^n$, the 
\emph{Laplace mechanism} outputs
$$\tilde{f}_L(d) = f(d) + (R_1, \ldots, R_K),$$
where $R_1, \ldots, R_K\sim\Lap(0, \frac{\gs_f}{\eps})$ are i.i.d.
random variables drawn from the 0-mean Laplace distribution with
scale $\frac{\gs_f}{\eps}$.

The Laplace mechanism is $(\eps, 0)$-DP. 
\label{thm:laplace}
\end{theorem}

\begin{theorem}[Exponential Mechanism~\cite{McSherryT07}]
Given an arbitrary range $\calR$, let $u: \calX^n \times \calR \rightarrow \reals$ be a utility function that maps database/output pairs to utility scores. Let $\gs_u = \max_r \gs_{u(\cdot, r)}$. For a fixed database $d \in \calX^n$ and privacy parameter $\eps > 0$, the \emph{exponential mechanism} outputs an element $r \in \calR$ with probability proportional to $\exp\left(\frac{\eps\cdot u(d, r)}{2 GS_u}\right)$.

The exponential mechanism is $(\eps, 0)$-DP.
\label{thm:exp-mech}
\end{theorem}

The following results allow us to use differentially private algorithms as building blocks in larger algorithms.

\begin{lemma}[Post-Processing~\cite{DMNS06}]
Let $M:\calX^n\rightarrow\calY$ be an $(\eps, \delta)$
differentially private and $f:\calY\rightarrow\calR$ be 
a (randomized) function. Then 
$f\circ M:\calX^n\rightarrow\calR$ is an
$(\eps, \delta)$
differentially private algorithm.
\label{lem:pp}
\end{lemma}

\begin{theorem}[Basic Composition~\cite{DMNS06}]
For any $k\in[K]$, let $M_k$ be an $(\eps_k, \delta_k)$ differentially private
algorithm. Then the composition of the $T$ mechanisms
$M = (M_1, \ldots, M_K)$
is $(\eps, \delta)$ differentially private where
$\eps = \sum_{k\in[K]}\eps_k$ and $\delta = \sum_{k\in[K]}\delta_k$.
\label{thm:comp}
\end{theorem}

\begin{definition}[Coupling]
Let $\bz$ and $\bz'$ be two random variables defined over the probability spaces $Z$ and $Z'$, respectively. A coupling of $\bz$ and $\bz'$ is a joint variable $(\bz_c, \bz_c')$ taking values in the product space $(Z \times Z')$ such that $\bz_c$ has the same marginal distribution as $\bz$ and $\bz'_c$ has the same marginal distribution as $\bz'$.
\end{definition}

\begin{definition}[$c$-Lipschitz randomized transformations]
\label{def:rand-lipschitz}
  A randomized transformation $T: \calX^n \rightarrow \mathcal{Y}^m$ is $c$-Lipschitz if for all datasets $d, d' \in \calX^n$, there exists a coupling $(\bz_c, \bz'_c)$ of the random variables $\bz = T(d)$ and $\bz' = T(d')$ such that with probability 1,
  \[
  H(\bz_c, \bz'_c) \leq c \cdot H(d, d')
  \]
 where $H$ denotes Hamming distance.
\end{definition}

\begin{lemma}[Composition with Lipschitz transformations (well-known)]
Let $M$ be an $(\eps, \delta)$-DP algorithm, and let $T$ be a $c$-Lipschitz transformation of the data with respect to the Hamming distance. Then, $M \circ T$ is $(c\eps, \delta)$-DP.
\label{lem:lipschitz}
\end{lemma}
\begin{proof}
Let $H(d, d')$ denote the distance (in terms of additions and removals or swaps) between datasets $d$ and $d'$.
By definition of the Lipschitz property, $H(T(d), T(d')) \leq c \cdot H(d, d')$. The lemma follows directly from the Lipschitz property on adjacent databases and the definition of $(\eps, \delta)$-differential privacy.
\end{proof}

\section{\texttt{NoisyStats}}

\subsection{Privacy Proof of \texttt{NoisyStats}}
\label{sec:gs}

\begin{lemma}
  We are given two vectors $\bx, \by\in[0, 1]^n$.
  Let \[\ncov(\bx, \by) = \bracs{\bx-\bar{x}\ones, \by-\bar{y}\ones} = (\sum_{i=1}^n x_i\cdot y_i) - \frac{(\sum_{i=1}^n x_i)(\sum_{i=1}^n y_i)}{n}\]
  and \[\nvar(\bx) = \bracs{\bx-\bar{x}\ones, \bx-\bar{x}\ones} = (\sum_{i=1}^n x_i^2) - \frac{(\sum_{i=1}^n x_i)^2}{n}.\] Also,
  let $\bar{x}, \bar{y}$ be the means of $\bx$ and $\by$ respectively and
  $\ones$ be the all ones vector.

  Then if $\gs_{\ncov}$ and $\gs_{\nvar}$ are the global sensitivities of functions $\ncov$
  and $\nvar$ then
  $\gs_{\ncov} = \left(1-\frac{1}{n}\right)$ and
  $\gs_{\nvar} = \left(1-\frac{1}{n}\right)$.
\label{lem:gs}
\end{lemma}

\begin{proof}
  Let $\bz=\bracs{\bx,\by}$ and $\bz'=\bracs{\bx',\by'}$ be neighbouring databases differing on the $n$th datapoint
\footnote{This is without loss of generality as we can always ``rotate'' both databases until the index on which they differ becomes the $n$th datapoint.}. Let $a=\sum_{i=1}^{n-1}x_i$ and $b=\sum_{i=1}^{n-1}y_i$ and note that $\max\{a,b\}\le n-1$. Then, 
\begin{align*}
\nvar(\bx)-\nvar(\bx')&=x_n^2-x_n'^2-\frac{2ax_n}{n}-\frac{x_n^2}{n}+\frac{2ax_n'}{n}+\frac{x_n'^2}{n}\\
&=(1-\frac{1}{n})(x_n^2-x_n'^2)+\frac{2a}{n}(x_n'-x_n).
\end{align*}
If $x_n'-x_n\le 0$ then $\nvar(\bx)-\nvar(\bx')\le (1-\frac{1}{n})(x_n^2-x_n'^2)\le 1-\frac{1}{n}$. Otherwise, 
\begin{align*}
\nvar(\bx)-\nvar(\bx')
&\le (1-\frac{1}{n})(x_n^2-x_n'^2)+\frac{2(n-1)}{n}(x_n'-x_n)\\
&=(1-\frac{1}{n})(x_n^2-2x_n+2x_n'-x_n'^2)
\end{align*}
Since $x_n\in[0,1]$ we have $x_n^2-2x_n\in[-1,0]$, so $\nvar(\bx)-\nvar(\bx')\le 1-\frac{1}{n}$. 

Also,
\begin{align*}
\ncov(\bx, \by)-\ncov(\bx',\by') \\
= x_ny_n-x_n'y_n'+ \\
\frac{a(y_n'-y_n)+b(x_n'-x_n)+x_n'y_n'-x_ny_n}{n}\\
\le (1-\frac{1}{n})(x_ny_n-x_n'y_n')+ \\
\frac{a(y_n'-y_n)+b(x_n'-x_n)}{n}
\end{align*}
If $y_n'-y_n\le0$ and $x_n'-x_n\le 0$ then $\ncov(\bx, \by)-\ncov(\bx',\by')\le (1-\frac{1}{n})(x_ny_n-x_n'y_n')\le (1-\frac{1}{n})$.
If $y_n'-y_n\le0$ and $x_n'-x_n>0$ then $\ncov(\bx, \by)-\ncov(\bx',\by')\le (1-\frac{1}{n})(x_ny_n-x_n'y_n'+(x_n'-x_n))$. Since $x_ny_n-x_n\le0$ and $x_n'-x_n'y_n'\le1$ we have $\ncov(\bx, \by)-\ncov(\bx',\by')\le 1-\frac{1}{n}$. Similarly if $y_n'-y_n>0$ and $x_n'-x_n\le0$ then $\ncov(\bx, \by)-\ncov(\bx',\by')\le 1-\frac{1}{n}$. Finally, if $y_n'-y_n>0$ and $x_n'-x_n>0$, we have 
\begin{align*}
\ncov(\bx, \by)-\ncov(\bx',\by') \le \\ (1-\frac{1}{n})(x_ny_n-x_n'y_n'+(y_n'-y_n)+(x_n'-x_n))\\
\le (1-\frac{1}{n})(x_n(y_n-1)-x_n'(y_n'-1)+(y_n'-1)-(y_n-1))\\
\le (1-\frac{1}{n})((x_n-1)(y_n-1)-(x_n'-1)(y_n'-1)).
\end{align*}
Since, $(x_n-1)(y_n-1)\in[0,1]$ and $(x_n'-1)(y_n'-1)\in[0,1]$, we have $\ncov(\bx, \by)-\ncov(\bx',\by')\le1-\frac{1}{n}$.
\end{proof}

\begin{proof}[Proof of Lemma~\ref{lem:ns}: (\texttt{NoisyStats})] 
  The global sensitivity of both $\ncov(\bx, \by)$ and $\nvar(\bx)$ is bounded by
  $\Delta = \left(1-1/n\right)$ (by Lemma~\ref{lem:gs}).

  As a result, if we sample $L_1, L_2 \sim \Lap(0,  3\Delta/\eps)$ then both
  $\ncov(\bx, \by) + L_1$ and $\nvar(\bx) + L_2$ are $(\eps/3, 0)$-DP
  estimates by the Laplace mechanism guarantees
  (see Theorem~\ref{thm:laplace}).
  By the post-processing properties of differential privacy
  (Lemma~\ref{lem:pp}), $1/(\nvar(\bx) + L_2)$ is a private release and the
  test $\nvar(\bx) + L_2 > 0$ is also private. As a result, $\tilde{\alpha}$
  is a $(2\eps/3, 0)$-DP release. Now to calculate the private intercept $\tilde\beta$, we use the
  global sensitivity of $(\bar{y} - \tilde{\alpha}\bar{x})$ which is at most
  $1/n\cdot(1+|\tilde\alpha|)$, since the means of $\bx, \by$ can change by at most $1/n$.
  \footnote{
  Alternatively, to estimate $\tilde\beta$, one could compute $\tilde{x}, \tilde{y}$, private estimates
  of $\bar{x}, \bar{y}$ by adding Laplace noise from $\Lap(0, 1/n)$ and then compute 
  $\hat\beta = \tilde{y} - \hat\alpha\tilde{x}$.
  }
  The Laplace noise we add ensures the private release of the intercept is $(\eps/3, 0)$-DP.

  Finally, by composition properties of differential privacy
  (Theorem~\ref{thm:comp}),
  Algorithm~\ref{alg:ns} is $(\eps, 0)$-DP.
\end{proof}

\subsection{On the Failure Rate of~\texttt{NoisyStats}}
\label{sec:nsfailures}

\begin{figure}[ht!]
\begin{subfigure}[b]{0.2\textwidth}
    \includegraphics[width=\textwidth]{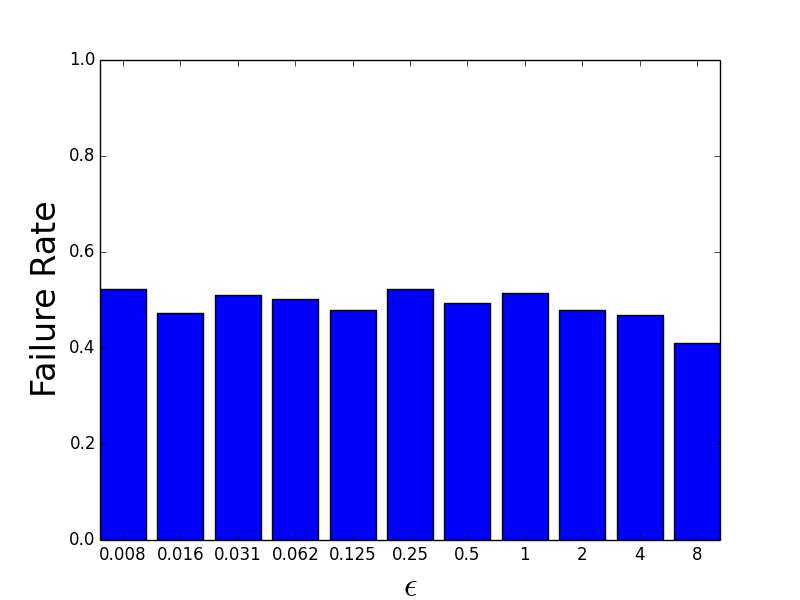}
    \caption{$\hat\alpha = 0.45577, \nvar(\bx) = 0.0245, n = 39$\\
    For Tract 800100 in County 31 in IL}
    \label{fig:IL_ns_faila}
  \end{subfigure}
  \qquad
  \begin{subfigure}[b]{0.2\textwidth}
    \includegraphics[width=\textwidth]{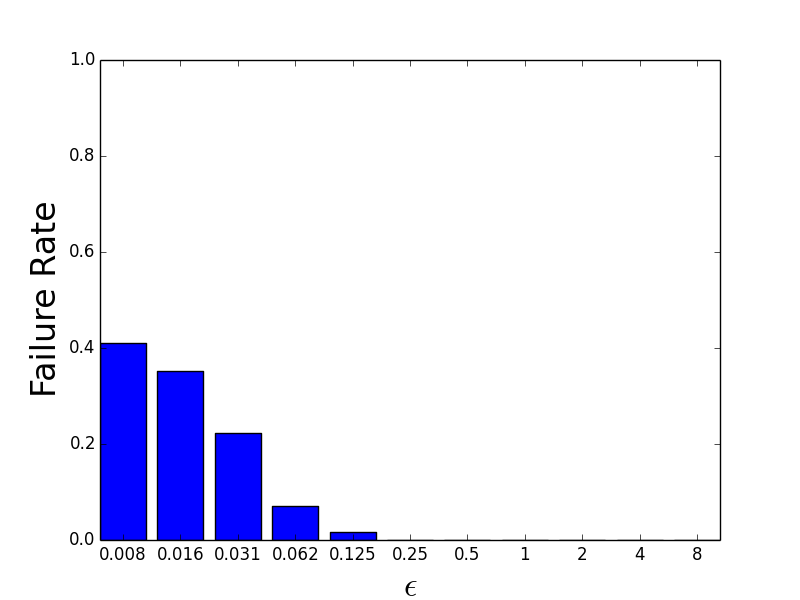}
    \caption{$\hat\alpha = 0.25217, \nvar(\bx) = 28.002, n = 381$\\For Tract 520100 in County 31 in IL}\label{fig:IL_ns_failb}
  \end{subfigure}
\noindent\caption{}
\end{figure}

\begin{figure}[ht!]
\begin{subfigure}[b]{0.2\textwidth}
    \includegraphics[width=\textwidth]{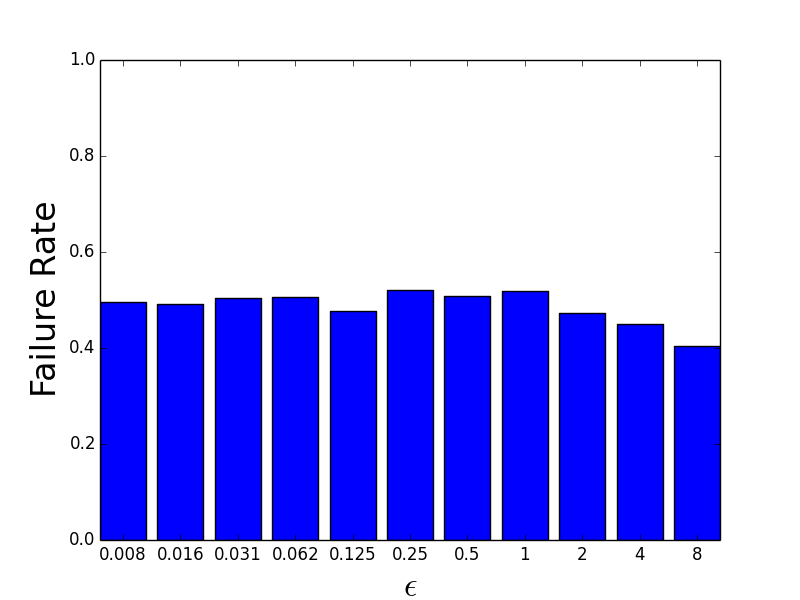}
    \caption{$\hat\alpha = 0.0965, \nvar(\bx) = 0.0245, n = 39$\\
    For Tract 2008 in County 63 in NC}\label{fig:NC_ns_faila}
  \end{subfigure}
  \qquad
  \begin{subfigure}[b]{0.2\textwidth}
    \includegraphics[width=\textwidth]{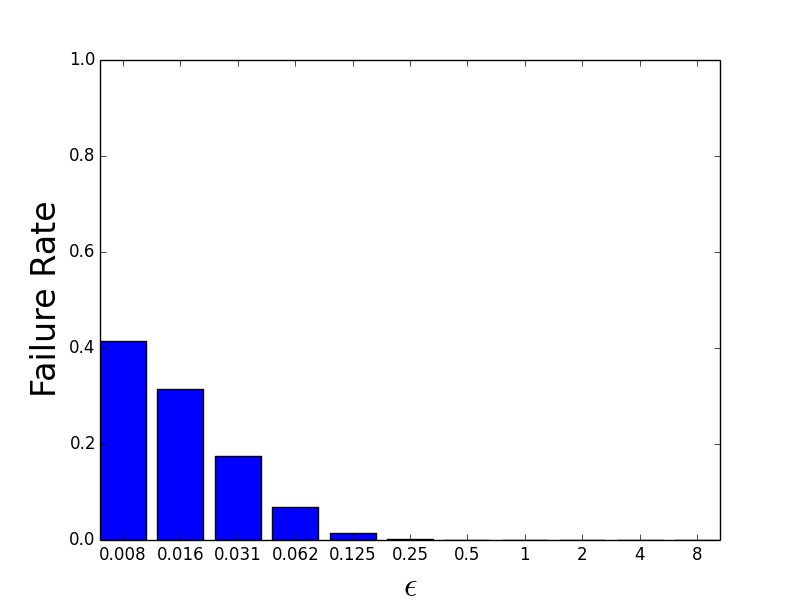}
    \caption{$\hat\alpha = 0.03757, \nvar(\bx) = 31.154, n = 433$\\
    For Tract 603 in County 147 in NC}\label{fig:NC_ns_failb}
  \end{subfigure}
  \caption{}
\end{figure}

\begin{figure}[ht!]
\begin{subfigure}[b]{0.3\textwidth}
    \includegraphics[width=\textwidth]{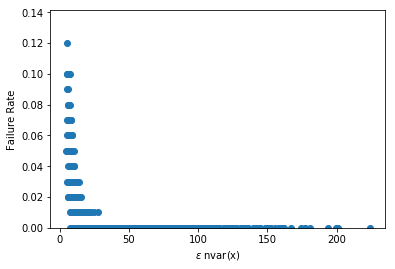}
\end{subfigure}
\caption{Failure rate of \texttt{NoisyStats} for all tracts in IL sorted by $\eps\nvar(\bx)$}\label{fig:IL_ns}
\end{figure}

Unlike all the other algorithms in this paper,~\texttt{NoisyStats} (Algorithm~\ref{alg:ns})
can fail by returning $\perp$. It fails if and only if the noisy $\nvar(\bx)$ sufficient statistic
becomes 0 or negative i.e. $\nvar(\bx) + L_2 \leq 0$ where $L_2\sim\Lap(0, 3(1-1/n)/\eps)$.
Since the statistic $\nvar(\bx)$ will always be non-negative, we also require the private version
of this statistic to be non-negative.
Intuitively, we see that $\nvar(\bx) + L_2$ is more likely to be less than or equal to 0 if
$\nvar(\bx)$ is small or $\eps$ is small. The setting of $\eps$ directly affects the standard deviation
of the noise distribution added to ensure privacy. The smaller $\eps$ is, the more spread out the
noise distribution is.
So we would expect that when $\eps$ or $\nvar(\bx)$
is small, Algorithm~\ref{alg:ns} would fail more often. We experimentally show this observation holds
on some tracts in IL and NC (from the semi-synthetic datasets given to us by the Opportunity Insights
team).

In Figures~\ref{fig:IL_ns_faila},~\ref{fig:IL_ns_failb}
,~\ref{fig:NC_ns_faila},~\ref{fig:NC_ns_failb}, we
see the failure rates for both high and low $\nvar(\bx)$ 
census tracts in both IL and NC as we vary $\eps$ between
$2^{-7}$ and $8$. We see that the failure rate is about 40\%
for any value of $\eps$ in low $\nvar(\bx)$ and is on average
less than 5\% for high $\nvar(\bx)$ tracts.
In Figure~\ref{fig:IL_ns}, we show
the failure rate for \texttt{NoisyStats} when evaluated on 
data from all tracts in IL. The results are averaged over
100 trials.
We see that the failure rate for $\eps = 8$ is 0\% for the majority of tracts. For tracts with
small $\nvar(\bx)$, the rate of failure is at most 12\%.
Thus, we can conclude that the failure rate approaches 0 as we increase either $\eps$ or $\nvar(\bx)$.

\section{\texttt{DPExpTheilSen} and \texttt{DPWideTheilSen}}

\subsection{Privacy Proofs for \texttt{DPExpTheilSen} and \texttt{DPWideTheilSen}}

\begin{lemma}\label{lem:match-k-lipschitz}
Let $T$ be the following randomized algorithm. For dataset $d = (x_i, y_i)_{i=1}^n$, let $K_n(d)$ be the complete graph on the $n$ datapoints, where edges denote points paired together to compute estimates in Theil-Sen. Then $K_n(d)$ can be decomposed into $n-1$ matchings, $\Sigma_1, \ldots, \Sigma_{n-1}$. Suppose $T(d)$ samples $k$ matchings without replacement from the $n-1$ matchings, and computes the corresponding pairwise estimates (up to $k n/2$ estimates). Then $T$ is a $k$-Lipschitz randomized transformation.
\end{lemma}
\begin{proof}
Let $\bz = T(d)$ and $\bz' = T(d')$ denote the multi-sets of estimates that result from applying $T$ to datasets $d$ and $d'$, respectively. 
  We can define a coupling $\bz_c$ and $\bz'_c$ of $\bz$ and $\bz'$.
  First, use $k$ matchings sampled randomly without replacement from $K_n(d)$, $\Sigma_1, \ldots, \Sigma_k$, to compute the multi-set of estimates $\bz_c = \{ z_{j,l}^{(p_{xnew})}: (x_j, x_l) \in \Sigma_1 \cup \ldots \cup \Sigma_k\} \}$. 
  Now, use the corresponding $k$ matchings from $K_n(d')$ to compute a multi-set of estimates $\bz_c' = \{ z_{j,l}^{(p_{xnew})}: (x'_j, x'_l) \in \Sigma_1 \cup \ldots \cup \Sigma_k \}$. This is a valid coupling because the $k$ matchings are sampled randomly without replacement from the complete graphs $K_n(d)$ and $K_n(d')$, respectively, matching the marginal distributions of $\bz$ and $\bz'$.

  Notice that every datapoint $x_j$ is used to compute exactly $k$ estimates in $\bz_c$. Therefore, for every datapoint at which $d$ and $d'$ differ, $\bz_c$ and $\bz'_c$ differ by at most $k$ estimates. Therefore, by the triangle inequality, we are done.
\end{proof}

\begin{proof}[Proof of Lemma~\ref{lem:ts-match}] 
  If DPmed$(z^{(p25)}, \eps, (n, k, \text{hyperparameters}))=\\\mathcal{M}(z^{(p25)}, \text{hyperparameters}))$ then Algorithm~\ref{alg:ts-match} is a composition of two algorithms, $\mathcal{M} \circ T$, where by Lemma~\ref{lem:match-k-lipschitz}, $T$ is a $k$-Lipschitz randomized transformation, and $\mathcal{M}$ is $(\eps/k, 0)$-DP.
  By the Lipschitz composition lemma (Lemma~\ref{lem:lipschitz}), Algorithm~\ref{alg:ts-match} (\texttt{DPTheilSenkMatch}) is $(\eps, 0)$-DP.
\end{proof}

\begin{proof}[Proofs of Lemmas~\ref{lem:ets} and~\ref{lem:wts}] The privacy of \texttt{DPExpTheilSenkMatch} and \texttt{DPWideTheilSenkMatch}
follows directly from Theorem~\ref{thm:exp-mech} and Lemma~\ref{lem:ts-match}.
\end{proof}

\subsection{Sensitivity of Hyperparameter}\label{sec:exphyperparameters}

Choosing optimal hyperparameters is beyond the scope of this work. However, in this section we present some preliminary work exploring the behavior of \texttt{DPWideTheilSen} with respect to the choice of $\theta$. In particular, we consider the question of how robust this algorithm is to the setting of the hyperparameter. Figure~\ref{fig:hyperparameters} shows the performance as a function the widening parameter $\theta$ on synthetic (Gaussian) data.  Note that in each graph both axes are on a log-scale so we see very large variation in the quality depending on the choice of hyperparameter. 

\begin{figure}[H]
\begin{subfigure}[b]{0.5\textwidth}
    \centering
    \includegraphics[width=0.7\textwidth]{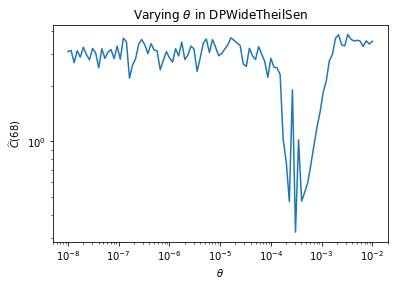}
    \label{fig:widening}
\end{subfigure}
\caption{Experimental results exploring the sensitivity of the hyperparameter choices for \texttt{DPWideTheilSen}. For each dataset $n=40$, and $n$ datapoints are generated as $x_i\sim \mathcal{N}(0,\sigma^2)$, $y_i=0.5*x_i+0.5+\mathcal{N}(0,\tau^2)$. The parameters of the data are fixed at $\sigma=10^{-3}$ and $\tau=10^{-4}$. The datapoints are then truncated so they belong between 0 and 1. Note that both axes are on a log scale.}
\label{fig:hyperparameters}
\end{figure}

\subsection{Pseudo-code for \texttt{DPExpTheilSen}}

In Algorithm~\ref{alg:em}, we give an efficient method for implementation of the DP median algorithm used as subroutine in \texttt{DPExpTheilSen}, the exponential mechanism for computing medians. 
To sample efficiently from this distribution, we implement a two-step algorithm following~\cite{C}: first, we sample an interval according to the exponential mechanism, and then we will sample an output uniformly at random from that interval. To efficiently sample from the exponential mechanism, we use the fact that sampling from the exponential mechanism is equivalent to choosing the value with maximum utility score after i.i.d. Gumbel-distributed noise has been added to the utility scores ~\cite{DworkTT014,ALT16}. 

\begin{algorithm}[h!]
  \KwData{$\bz$}
  \KwPrivacyparams{$\eps$}
  \KwInput{$n, k, r_l, r_u$}
  
  $\eps = \eps/k$
  
  Sort $\bz$ in increasing order
  
  Clip $\bz$ to the range $[r_l, r_u]$
  
  Insert $r_l$ and $r_u$ into $\bz$ and set $n = n+2$
  
  Set $\textrm{maxNoisyScore} = -\infty$
  
  Set $\textrm{argMaxNoisyScore} = -1$

  \For{$i \in [1, n)$} {
    $\textrm{logIntervalLength} = \log(\bz[i] - \bz[i-1])$
    
    $\textrm{distFromMedian} = \lceil| i - \frac{n}{2} | \rceil$
    
    $\textrm{score} = \textrm{logIntervalLength} -\frac{\eps}{2} \cdot \textrm{distFromMedian}$
    
    $N \sim \textrm{Gumbel}(0,1)$ 
    
    $\textrm{noisyScore} = \textrm{score} + N$
    
    \If {$\textrm{noisyScore} > \textrm{maxNoisyScore}$} {
        $\textrm{maxNoisyScore}= \textrm{noisyScore}$
        
        $\text{argMaxNoisyScore} = i$
        }
  }
  
  $\text{left} = \bz[\text{argMaxNoisyScore-1}]$
  
  $\text{right} = \bz[\text{argMaxNoisyScore}]$
  
  Sample $\tilde{m} \sim \text{Unif}\left[\text{left} , \text{right} \right]$
  
  \Return $\tilde{m}$

  \caption{Exponential Mechanism for Median:
  $(\eps/k, 0)$-DP Algorithm}  \label{alg:em}
\end{algorithm}

\subsection{Pseudo-code for \texttt{DPWideTheilSen}}

The pseudo-code for \texttt{DPWideTheilSen} is given in Algorithm~\ref{alg:wem}. It is a small variant on Algorithm~\ref{alg:em}.

\begin{algorithm}[h!]
  \KwData{$\bz$}
  \KwPrivacyparams{$\eps$}
  \KwInput{$n, k, \theta, r_l, r_u$}
  
  $\eps = \eps/k$
  
  Sort $\bz$ in increasing order
  
  Clip $\bz$ to the range $[r_l, r_u]$
  
  \If {$n$ is even} {
    Insert $m$, the true median, into $\bz$
    
    Set $n = n+1$
  }
  
  \For{$i \in [0, \lfloor \frac{n}{2} \rfloor]$}{
    $z[i] = \max(z_l, z[i] - \theta)$
    
    $z[n-i-1] = \min(z_u, z[i] + \theta)$
  }
  
  Insert $z_l$ and $z_b$ into $\bz$ and set $n = n+2$
  
  Set $\textrm{maxNoisyScore} = -\infty$
  
  Set $\textrm{argMaxNoisyScore} = -1$

  \For{$i \in [1, n)$} {
    $\textrm{logIntervalLength} = \log(\bz[i] - \bz[i-1])$
    
    $\textrm{distFromMedian} = \lceil| i - \frac{n}{2} | \rceil$
    
    $\textrm{score} = \textrm{logIntervalLength} -\frac{\eps}{2} \cdot \textrm{distFromMedian}$
    
    $N \sim \textrm{Gumbel}(0,1)$ 
    
    $\textrm{noisyScore} = \textrm{score} + N$
    
    \If {$\textrm{noisyScore} > \textrm{maxNoisyScore}$} {
        $\textrm{maxNoisyScore}= \textrm{noisyScore}$
        
        $\text{argMaxNoisyScore} = i$
        }
  }
  
  $\text{left} = \bz[\text{argMaxNoisyScore-1}]$
  
  $\text{right} = \bz[\text{argMaxNoisyScore}]$
  
  Sample $\tilde{m} \sim \text{Unif}\left[\text{left} , \text{right} \right]$
  
  \Return $\tilde{m}$

  \caption{$\theta$-Widened Exponential Mechanism for Median:  
  $(\eps/k, 0)$-DP Algorithm}  \label{alg:wem}
\end{algorithm}

\section{\texttt{DPSSTheilSen}} \label{sec:compute-ss}

Suppose we are given a dataset $(\bx, \by)$. Consider a neighboring dataset $(\bx', \by')$
that differs from the original dataset in exactly one row.
Let $\bz$ be the set of point estimates (e.g. the $p25$ or $p75$ point estimates) induced by the dataset $(\bx, \by)$, and let $\bz'$ be the set of point estimates induced by dataset $(\bx', \by')$ by Theil-Sen. Formally, for $N = k n/2$,
we let $\calZ_k:[0, 1]^n\times[0, 1]^n\rightarrow\reals^N$ denote
the function that transforms a set of point coordinates into estimates for each pair of 
points.
Then $\bz = \calZ(\bx, \by)$, $\bz' = \calZ(\bx', \by')$.
Notice that changing one datapoint in $(\bx, \by)$ changes at most $k$ of the point estimates in $\bz$. Assume that both $\bz$ and $\bz'$ are in sorted order.
Recall the definition of $S_{\text{med}\circ\calZ, t}^k((\bx, \by))$:
\begin{align*} 
    S_{\text{med}\circ\calZ_k, t}^k((\bx, \by))& \\
    =\max\Big\{&z_{m+k}-z_m, z_{m}-z_{m-k}, \\
    & \max_{l = 1, \ldots, n} \max_{s=0, \cdots, k(l+1)} e^{-lt} ( z_{m+s}-z_{m-(k(l+1)+s})\Big\},
\end{align*} 

Let \[LS_{\text{med}}^k(\bz)=\max_{\bz'\in\mathbb{R}^N, \text{Ham}(\bz, \bz')\le k}|\text{med}(\bz)-\text{med}(\bz')|\] be distance $k$ local sensitivity of the dataset $\bz$ with respect to the median. 
In order to prove that $S_{\text{med}\circ\calZ_k, t}^k((\bx, \by))$ is a $t$-smooth upper bound on $LS_{\text{med}\circ\calZ_k}$, we will use the observation that \[LS_{\text{med}\circ\calZ_k}(\bx, \by)\le LS_{\text{med}}^k(\bz).\] Now Figure~\ref{SSn-1proof} outlines the maximal changes we can make to $\bz$. For $l\ge1$ and any interval of $lk+k+1$ points containing the median, we can move the median to one side of the interval by moving $kl$ points, and to the other side by moving an additional $l$ points. Therefore, for $l\ge 1$, 
\begin{equation}\label{LSSS}
\max_{\bz': d(\bz, \bz') \leq lk}
    LS_{\text{med}}(\bz') = \max_{s=0, \cdots, lk+k} \{ z_{m+s}-z_{m-(lk+k)+s} \}
\end{equation}
so 
\[S_{\text{med}, t}^k(\bz) = \max_{l = 0, \ldots, n} e^{-lt} \max_{\bz': d(\bz, \bz') \leq lk}
    LS^{k}_{\text{med}}(\bz'). \]

\begin{figure}
\includegraphics[scale=0.4]{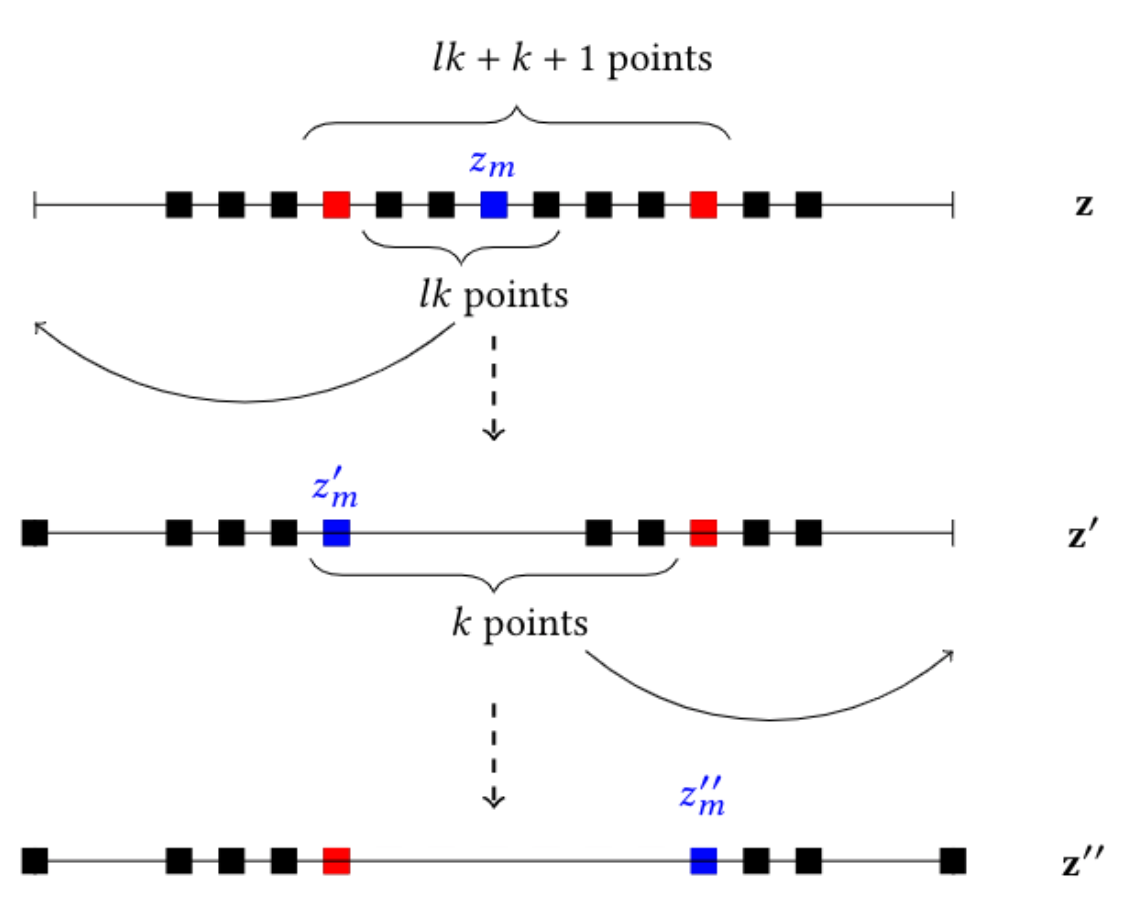}
\caption{A brief proof by pictures of Equation~\ref{LSSS}.
}
\label{SSn-1proof}
\end{figure}

\begin{proof}[Proof of Lemma~\ref{lem:ss-upper}]
We need to show that $S_{\text{med}\circ\calZ_k, t}^k((\bx, \by))$ is lower bounded by the local sensitivity and that
for any dataset $(\bx', \by')$ such that $d((\bx, \by), (\bx', \by')) \leq l$, we have $S_{\text{med}\circ\calZ_k, t}^k((\bx, \by)) \leq e^{tl} S_{\text{med}\circ\calZ_k, t}^k((\bx', \by'))$.

By definition of $S_{\text{med}, t}^k$, we see that 
$S_{\text{med}\circ\calZ_k, t}^k((\bx, \by))\geq LS_{\text{med}}^k$ (e.g. when $l = 0$ in the formula for 
$S_{\text{med}, t}^k$).
Next, we see that
\begin{align}
S_{\text{med}, t}^k(\bz) &= \max_{l = 0, \ldots, n} e^{-lt} \max_{\bz': d(\bz, \bz') \leq lk}
    LS^{k}_{\text{med}}(\bz') \\
    &\leq e^t \cdot \max_{l = 1, \ldots, n} e^{-lt} \max_{\bz'': d(\bz', \bz'') \leq lk}
    LS_{\text{med}}^k(\bz'') \\
    &\leq e^t\cdot S_{\text{med}, t}^k(\bz'),
\end{align}
which completes our proof.
\end{proof}

\begin{lemma}  
Let $M(\bx) = \text{median}(\bx) + \frac{1}{s} S_{\text{med}, t}^k(\bx) \cdot N$, where $N, s$ and $t$ are computed according to Algorithm~\ref{alg:st}.
Then, $M$ is $(\eps, 0)$-DP.
\label{lem:ss}
\end{lemma}
\begin{proof}  
Let $D_{\infty}(P \vert\vert Q) = \sup_{x \in \text{supp}(Q)} \log \frac{p(x)}{q(x)}$ denote the max-divergence for distributions $P$ and $Q$. Let $N$ be a random variable sampled from StudentsT$(d)$, where $d > 0$ is the degrees of freedom. From Theorem 31 in~\cite{BunS19}, we have that for $s, t > 0$,
\begin{align*}
\left.
\begin{array}{ll}
D_\infty(N \vert\vert e^t N + s) \\
D_\infty(e^t N + s \vert\vert N)
\end{array}
\right\}
\leq
|t| (d+1) + |s| \cdot \frac{d+1}{2\sqrt{d}}
\end{align*}
The parameters $s$ and $t$ correspond to the translation (shifting) and dilation (scaling) of the StudentsT$(d)$ distribution.

Setting $s = 2 \sqrt{d} \left( \frac{\eps' - |t|(d+1)}{d+1} \right)$ as in Algorithm~\ref{alg:st}, we have that for $|t|(d+1) < \eps'$,
\begin{align} \label{eq:st-divergence}
\begin{cases}
D_\infty(N \vert\vert e^t N + s) \\
D_\infty(e^t N + s \vert\vert N)
\end{cases} 
\leq
\eps
\end{align}
If Equation~\ref{eq:st-divergence} is satisfied, then by Theorem 46 in~\cite{BunS19}, the mechanism in Algorithm~\ref{alg:st}, $M(\bz) = \text{median}(\bz) + \frac{1}{s} S^t_{\text{median}(\cdot)}(\bz) + N$, is $(\eps, 0)$-DP.
\end{proof}

\begin{algorithm}[!ht]
  \KwData{$\bz, \data$}
  \KwPrivacyparams{$\eps$}
  \KwHyperparams{$k, n, r_l, r_u, d$}
  
  Set $t = \frac{\eps}{2(d+1)}$ and $s = \frac{\eps \sqrt{d}}{d+1}$
 
  $S_{\textrm{median}} = S_{\text{med}, t}^k((\bx, \by))$ 
  
  Sample $N \sim \textrm{Student's T}(d)$
  
  Set $\widetilde{m} = \textrm{median}(\bz) + \frac{1}{s} \cdot S_{\textrm{median}} \cdot N$
  
  \Return $\widetilde{m}$
  \caption{Smooth Sensitivity Student's T Noise Addition for Median: $(\eps, 0)$-DP Algorithm  \label{alg:st}}
\end{algorithm}

\subsection{Sensitivity of Hyperparameters}

In Algorithm~\ref{alg:st} we set the smoothing parameter to be a specific function of $\epsilon$ and $d$: $t=\frac{\epsilon}{2(d+1)}$ and $s=\frac{\epsilon\sqrt{d}}{d+1}$. There were other choices for these parameters. For any $\beta\in[0,1]$, the $(\epsilon, 0)$-DP guarantee is preserved if we set \[t=\frac{\epsilon\beta}{d+1} \text{\;\; and \;\;} s=2\sqrt{d}\left(\frac{\epsilon - t (d+1)}{d+1}\right).\] Algorithm~\ref{alg:st} corresponds to setting $\beta=1/2$. Increasing $\beta$ increases $t$, which results in $S_{\text{med}\circ\calZ_k}^k((\bx, \by))$ decreasing. However, if increasing $\beta$ also decreases $s$. In Figure~\ref{fig:betahyperparam} we explore the performance of \texttt{DPSSTheilSen} as a function of $\beta$ on synthetic (Gaussian) data. Note that the performance doesn't seem too sensitive to the choice of $\beta$ and $\beta=0.5$ is a good choice on this data.

\begin{figure}[ht!]
    \centering
    \includegraphics[scale=0.5]{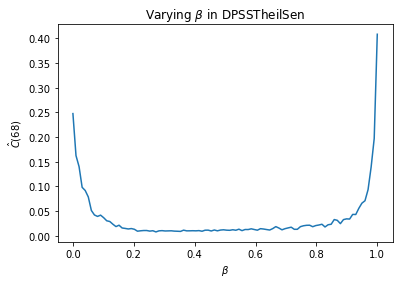}
    \caption{Experimental results exploring the sensitivity of the hyperparameter choices for \texttt{DPSSTheilSen}. For each dataset $n=30$, and $n$ datapoints are generated as $x_i\sim \mathcal{N}(0.5,\sigma^2)$, $y_i=0.5*x_i+0.2+\mathcal{N}(0,\tau^2)$. The parameters of the data are fixed at $\sigma=0.01$ and $\tau=0.005$. The datapoints are then truncated so they belong between 0 and 1. Note that both axes are on a log scale.}
    \label{fig:betahyperparam}
\end{figure}

\section{\texttt{DPGradDescent}}\label{appendixDPgraddescent}

There are three main versions of \texttt{DPGradDescent} we consider:
\begin{enumerate}
\item \texttt{DPGDPure}: $(\eps, 0)$-DP;
\item \texttt{DPGDApprox}: $(\eps, \delta)$-DP; and
\item \texttt{DPGDzCDP}: $(\eps^2/2)$-zCDP.
\end{enumerate}

Algorithm~\ref{alg:dpgdpure} is the specification of a $(\eps, 0)$-DP
algorithm and Algorithm~\ref{alg:dpgdz} is
a $\rho$-zCDP algorithm, from which we can obtain a $(\eps, \delta)$-DP algorithm
and a $(\eps^2/2)$-zCDP algorithm.
As with traditional gradient descent, there are several choices that have been made in designing this algorithm: the step size, the batch size for the gradients, how many of the estimates are averaged to make our final estimate, how the privacy budget is distributed. We have included this pseudo-code for completeness to show the choices that were made in our experiments. We do not claim to have extensively explored the suite of parameter choices, and it is possible that a better choice of these parameters would result in a better performing algorithm. Differentially private gradient descent has received a lot of attention in the literature. 
For a more in-depth discussion of DP gradient descent see \cite{BST14}.

\begin{algorithm}
  \KwData{\data}
  \KwPrivacyparams{$\eps$}
  \KwHyperparams{$n, T, \tau, \tilde{p}_{25}^0, \tilde{p}_{75}^0$}

  \For{$t=0:T-1$}{
  $\eps_t = \eps/T$
  
  \For{$i=1:n$}{
  
  $\tilde{y}_{i,t} = 2(\tilde{p}_{25}^t*(3/4-x_i)+\tilde{p}_{75}^t(x_i-1/4))$
  
  $\Delta_{i,t} = \begin{pmatrix} [2(y_i-\tilde{y})(3/4-x_i)]_{-\tau}^{\tau}, \\ [2(y_i-\tilde{y})(x_i-1/4)]_{-\tau}^{\tau}\end{pmatrix}$
  }
  
  $\Delta_t = \sum_{i=1}^n \Delta_{i,t} + \Lap_2\left(0, 4\tau/\eps_t\right)$
  
  $\gamma_t = \frac{1}{\sqrt{\sum_{l=0}^t \Delta_l^2}}$
  
  $[\tilde{p}_{25}^{t+1}, \tilde{p}_{75}^{t+1}] = [\tilde{p}_{25}^{t}, \tilde{p}_{75}^{t}]-\gamma_t*\Delta_t$
  }
  
  \Return $\frac{2}{T}\sum_{t=T/2}^{T-1} [\tilde{p}_{25}^{t}, \tilde{p}_{75}^{t}]$
  
  \caption{\texttt{DPGDPure}: $(\eps, 0)$-DP Algorithm}\label{alg:dpgdpure}
\end{algorithm}

\begin{algorithm}
  \KwData{\data}
  \KwPrivacyparams{$\rho$}
  \KwHyperparams{$n, T, \tau, \tilde{p}_{25}^0, \tilde{p}_{75}^0$}

  \For{$t=0:T-1$}{
  $\rho_t = \rho/T$
  
  \For{$i=1:n$}{
  
  $\tilde{y}_{i,t} = 2(\tilde{p}_{25}^t*(3/4-x_i)+\tilde{p}_{75}^t(x_i-1/4))$
  
  $\Delta_{i,t} = \begin{pmatrix} [2(y_i-\tilde{y})(3/4-x_i)]_{-\tau}^{\tau}, \\ [2(y_i-\tilde{y})(x_i-1/4)]_{-\tau}^{\tau}\end{pmatrix}$
  }
  
  $\Delta_t = \sum_{i=1}^n \Delta_{i,t} + \calN_2\left(0, (2\tau/\sqrt\rho_t)^2\right)$
  
  $\gamma_t = \frac{1}{\sqrt{\sum_{l=0}^t \Delta_l^2}}$
  
  $[\tilde{p}_{25}^{t+1}, \tilde{p}_{75}^{t+1}] = [\tilde{p}_{25}^{t}, \tilde{p}_{75}^{t}]-\gamma_t*\Delta_t$
  }
  
  \Return $\frac{2}{T}\sum_{t=T/2}^{T-1} [\tilde{p}_{25}^{t}, \tilde{p}_{75}^{t}]$
  
  \caption{\texttt{DPGDzCDP}: $\rho$-zCDP Algorithm}\label{alg:dpgdz}
\end{algorithm}

\begin{lemma}
For any $\rho > 0$, Algorithm~\ref{alg:dpgdz} is $\rho$-zCDP.
\end{lemma}

\begin{proof}
By composition properties of zCDP, it suffices to show that
$\Delta_t = \sum_{i=1}^n \Delta_{i,t} + \calN_2\left(0, (2\tau/\sqrt\rho_t)^2\right)$ is
$\rho_t$-zCDP
where $\calN_2(0, s^2)$ represents two 
(independent) draws from a Normal distribution with standard deviation $s$.

The $L_2$-sensitivity of $\sum_{i=1}^n \Delta_{i,t}$ is at most $2\sqrt{2}\tau$ and the standard deviation
of the Gaussian distribution from which noise is added is $2\tau/\sqrt{\rho_t}$. Then by 
Proposition 1.6 in~\cite{BunS16}, the procedure to compute $\Delta_t$ is $\rho_t$-zCDP.

\end{proof}

\begin{lemma}
For any $\delta\in (0, 1]$ and any $\rho > 0$, Algorithm~\ref{alg:dpgdz} is
$(\eps, \delta)$-DP where
$$\eps = \rho + \sqrt{4\rho\log\left(\frac{\sqrt{\pi\rho}}{\delta}\right)}.$$
\end{lemma}

\begin{proof}
Follows from Lemma 3.6 in~\cite{BunS16}.
\end{proof}

\begin{lemma}
For any $\eps > 0$, Algorithm~\ref{alg:dpgdpure} is $(\eps, 0)$-DP.
\end{lemma}

\begin{proof}
By basic composition, it suffices to show that
$\Delta_t = \sum_{i=1}^n \Delta_{i,t} + \Lap_2\left(0, 4\tau/\eps_t\right)$
is $(\eps_t, 0)$-DP where $\Lap_2(0, s)$ represents two 
(independent) draws from a Laplace distribution with scale $s$.
This holds since
the $L_1$-sensitivity of $\sum_{i=1}^n \Delta_{i,t}$ is at most $4\tau$.
\end{proof}

\end{document}